\documentclass[11pt]{article}

\usepackage{format}
\usepackage{notation}

\usepackage{mdframed}

\title{Simple Opinion Dynamics for No-Regret Learning}

\author{%
  John Lazarsfeld \\
  SUTD \\
  \texttt{jlazarsfeld@gmail.com}
  \and
  Dan Alistarh \\
  IST Austria \\
  \texttt{dan.alistarh@ist.ac.at}
}

\date{}

\begin{document}

\maketitle
\thispagestyle{empty}

\begin{abstract}
We study a cooperative multi-agent bandit
setting in the distributed GOSSIP model:
in every round, each of $n$ agents chooses an action from
a common set, observes the action's corresponding reward,
and subsequently exchanges information with
\textit{a single} randomly chosen neighbor,
which may inform its choice in the next round. 
We introduce and analyze families of memoryless
and time-independent protocols for this setting,
inspired by opinion dynamics that are well-studied
for other algorithmic tasks in the GOSSIP model. 

For stationary reward settings, we prove for the first time that these 
simple protocols exhibit \textit{best-of-both-worlds}
behavior, simultaneously obtaining \textit{constant}
cumulative regret scaling like $R(T)/T = \widetilde O(1/T)$,
and also reaching consensus on
the highest-mean action within $\widetilde O(\sqrt{n})$ rounds.
We obtain these results by showing a new connection
between the global evolution of these
decentralized protocols and a class of
\textit{zero-sum multiplicative weights update} processes.
Using this connection, we establish a general framework
for analyzing the population-level regret and other
properties of our protocols.  
Finally, we show our protocols are also surprisingly
robust to adversarial rewards, and in this regime we
obtain sublinear regret
scaling like $R(T)/T = \widetilde O(1/\sqrt{T})$ as long as the
number of rounds does not grow too fast 
as a function of $n$. 
\end{abstract}

\newpage
\setcounter{page}{1}

% sections
% Introduction

\section{Introduction}

The multi-armed bandit problem, where a single
learning agent chooses actions over a sequence of rounds
in order to maximize its total reward,
is among the most well-studied in online learning.
Distributed, multi-agent variants
of this problem have also been widely studied under
various constraints; one particular such line of work is the
\textit{cooperative multi-agent bandit setting},
where agents are connected over a communication graph
and play against a common bandit instance,
choosing actions in parallel over $T$ rounds.
Each agent locally runs a bandit algorithm that may involve
communication with neighbors,
and the information exchanged can be used to
determine an agent's future actions.
This cooperative setting has been studied for both
stochastic~\cite{szorenyi2013gossip,landgren2016distributed,
  kolla2018collaborative, martinez2019decentralized,
  DBLP:journals/pomacs/SankararamanGS19}
and non-stochastic bandits~\cite{awerbuch2008online,
  cesa2016delay, bar2019individual},
where communication between agents has been shown to
improve an agent's regret on average,
compared to each agent locally running a centralized bandit
algorithm without any communication.

However, most prior works in this setting require 
that every agent communicate with \textit{all its neighbors}
in each round
(as pointed out by Cesa-Bianchi et al.~\cite{cesa2016delay},
this resembles the LOCAL model of
distributed computation~\cite{linial1992locality}).
When the underlying graph is dense, this volume of
communication may be prohibitively large,
which is a known bottleneck in many practical
settings, including in  distributed machine
learning~\cite{alistarh2017qsgd,
  koloskova2019decentralized,
  DBLP:conf/icml/KoloskovaSJ19}.

In contrast, much less is known about
cooperative multi-agent bandits in more
lightweight \emph{decentralized} models of
distributed communication, such as the
GOSSIP model~\cite{boyd2006randomized,shah2009gossip}.
In this model, at every round, each agent is
randomly connected to \textit{one of its neighbors},
and thus the total number of information
exchanges per-round scales only linearly in the size of
the population, even for dense communication graphs.
Algorithms for general distributed tasks have been
studied extensively in the GOSSIP model,
both in modern machine learning and optimization
settings~\cite{lian2017can,
  assran2019stochastic,
  DBLP:conf/nips/EvenBBFHGMT21},
and also in the context of simple, memoryless
\textit{opinion dynamics}, where simple
local transition rules can be shown to solve
more complex global tasks like consensus and
majority~\cite{
  10.1214/11-PS184,
  10.3150/12-BEJSP04,
  chierichetti2010rumour,
  bastide2021self,
  becchetti2020consensus,
  becchetti2024minority}.

Yet, much less is known about the power of
this model in the cooperative bandit setting,
and in particular on the learning properties of
simple opinion-dynamic-like protocols
that have been studied in other contexts.
This line of inquiry is the focus of the present paper,
and our main results prove that simple protocols of
this type can obtain both sublinear regret
and reach consensus on the best action
in this setting.
For concreteness, we begin by introducing the
setting more precisely:

\subsection{Problem Setting}
\label{sec:intro:setting}

\begin{setting*}
  Consider $n$ anonymous agents distributed over a
  communication graph $G$, and
  interacting with an $m$-armed bandit
  over $T$ rounds. In each round $t$:
  \begin{enumerate}[
    label=(\roman*),
    topsep=0.25em,
    itemsep=-0.25em,
    ]
  \item
    Every agent $u \in [n]$
    \textit{chooses an action} $j \in [m]$.
  \item
    Each action $j \in [m]$ generates a reward
    $g^t_j \sim \nu^t_j$.
    Each agent choosing action $j$
    observes $g^t_j$.
  \item
    Every agent pulls information
    from a \textit{single neighbor}
    sampled uniformly at random. 
  \end{enumerate}
\end{setting*}
In this setting, we assume each agent is equipped
with a fixed local (possibly randomized) protocol
$\calP$. Based on an agent's observed rewards and
any information exchanged with its neighbor,
the protocol $\calP$ determines its action choice
at the next round.
When every agent uses the protocol $\calP$
to choose its actions,
we measure its global performance via a
\textit{population-level} regret.
For this, we let
$\pt := (p^t_1, \dots, p^t_m) \in \Delta_m$
denote the discrete distribution whose coordinates
specify the \textit{fraction} of $n$ agents
choosing each action $j$ at round $t$.
At time $t=0$, we assume each agent is deterministcally
initialized with an action choice $j$
such that $\p^0 = (1/m, \dots, 1/m)$.
Moreover, we denote
by $\gt := (g^t_1, \dots, g^t_m)$ the vector
of realized rewards at round $t$.
Then we define:

\begin{definition}[\textbf{Population-level Regret}]
  \label{def:regret}
  Given a time horizon $T$ and
  a reward sequence $\{\gt\}$,
  the regret of the sequence $\{\pt\}$
  generated by $\calP$
  is  
  $R(T) :=
  \max_{j \in [m]}\;
  {\sum_{t \in [T]}} 
  \E[g^t_j] - 
  \E[\langle\pt, \gt\rangle]
  $.
\end{definition}
In words, $R(T)$ measures difference
between the cumulative expected reward of
the best \textit{fixed} action
and the expected cumulative reward of the population
\textit{averaged} over all agents
(i.e., weighted at each round by the distribution $\pt$).
We make the following remarks:
\begin{itemize}[
  topsep=0.5em,
  leftmargin=1.5em,
  itemsep=0em,
  ]
\item
  In the problem setting,
  each agent (from a local perpsective) receives
  only bandit feedback on the reward of its chosen
  action, while from a global perspective,
  the full reward vector $\gt$ can be distributed
  across the population.
  Globally, the problem can thus be
  viewed as a decentralized instance
  of the \textit{prediction with expert advice}
  (\textit{experts}) setting from online learning
  \cite{cesa2006prediction}.
  This is similar to the model of other prior
  works on cooperative bandits
  \cite{cesa2016delay, bar2019individual}.
  
\item
  The regret definition $R(T)$ is sometimes referred to
  as the \textit{average welfare regret}
  or \textit{social regret} and is the same
  notion considered in related works~
  \cite{landgren2016distributed,
    cesa2016delay, DBLP:conf/podc/CelisKV17,
    martinez2019decentralized, bar2019individual}.
  The gold standard in online learning and bandit
  settings is to ensure $R(T)$ grows sublinearly 
  in $T$ (meaning $R(T) = o(T)$, 
  or equivalently $R(T)/T = o(1)$). If a protocol $\calP$
  achieves this goal, we say $\calP$ \textit{obtains sublinear
    regret}, or that $\calP$ is a \textit{no-regret
    learning} protocol~\cite{bubeck2012regret, lattimore2020bandit}. 
\end{itemize}  
The primary focus on this work is the
\textit{stationary reward setting}, where
the distributions $\nu^t_j$ remain fixed
over rounds (this is sometimes referred to as the
\textit{stochastic rewards} setting, and is in contrast
to an \textit{adversarial rewards} setting, where
reward means may change over time).
Other than several standard boundedness conditions,
we make no other distributional assumptions.
Defined formally:

\begin{restatable}[\textbf{Stationary Rewards}]
  {assumption}{stationaryrewards}
  \label{reward:stationary}
  For each $j \in [m]$, assume for all rounds $t \in [T]$
  that $g^t_j \sim \nu_j$, where
  for some $\sigma \ge 1$, 
  $\nu_j$ has support $[0, \sigma]$ and mean
  $\mu_j := \E[\nu_j] \in [0, 1]$.
  We assume each $\mu_j$ is an absolute constant,   
  and without loss of generality, we assume
  $1 \ge \mu_1 > \mu_2 \ge \dots \ge \mu_m \ge 0$.
\end{restatable}

\paragraph{Algorithmic focus:
  learning via simple opinion dynamics}
Most prior works on cooperative bandits
under the more powerful LOCAL communication model
employ adaptations of common
single-agent bandit algorithms (such as UCB and
EXP3~\cite{bubeck2012regret})
to the multi-agent setting. 
While this approach can lead to 
sublinear regret bounds in both stationary
and adversarial reward settings,
these adaptations usually require agents to
locally maintain a probability
distribution over the full set of actions at each round,
and to carefully aggregate information
received from multiple neighbors in order to determine
future choices~\cite{
  landgren2016distributed,
  kolla2018collaborative,martinez2019decentralized,
  cesa2016delay, bar2019individual}.

In contrast, the main algorithmic focus of this work
is to understand whether extremely simple protocols
related to \textit{opinion dynamics} --- studied extensively
for other problems in the GOSSIP model ---
can obtain no-regret guarantees in this learning setting.
In particular, in the general GOSSIP model,
opinion dynamics are local protocols in which
an agent's current state (opinion) is updated
based only on the state of the agent it
interacted with in the most recent round
\cite{becchetti2020consensus}.
These dynamics model fully decentralized scenarios in which
agents are anonymous, extremely computationally limited,
and are only able to employ some fixed, time-independent update rule.

Despite their simplicity, opinion dynamics
have been shown to exhibit complex
global behavior, and many recent works have
proven their ability to solve
consensus~\cite{doerr2011stabilizing,
  becchetti2024minority},
majority~\cite{becchetti2014plurality,
  ghaffari2016polylogarithmic},
and other synchronization tasks~\cite{bastide2021self}
important in distributed and decentralized settings.
Moreover, opinion-dynamic like processes
are useful primitives in other decentralized learning
and optimization settings~\cite{lu2011gossip,lian2017can,
  DBLP:conf/icml/KoloskovaSJ19},
and analyzing such processes in the present
setting serves as an intial step in developing
algorithms robust to  systems prone to communication errors,
and with interchangeable agents entering and leaving
the population over time~\cite{martinez2019decentralized}.

From this perspective, in the present cooperative
bandit setting, it is natural to view an agent's action
choice at round $t$ as its current opinion,
and this opinion (action choice) can be updated in the
subsequent round based only on the information exchanged
with its randomly sampled neighbor. 
This raises the natural question of whether,
in the cooperative bandit setting,
simple protocols of this form can effectively learn
at a population level.
To make this precise, we define the following
\textit{memoryless and time-independent}
property to describe protocols with
opinion-dynamic-like attributes:

\begin{restatable}{property}{memoryless}
  \label{property:memoryless}
  A protocol $\calP$ is
  \textbf{memoryless and time-independent} if
  (i) an agent's action choice at round $t+1$
  depends only on its own action choice
  and observed reward from round $t$,
  and those of its random interaction partner,
  and (ii) the decision rule
  used by an agent is identical in all rounds.
\end{restatable}

\noindent
In other words, using a memoryless and time-independent
protocol means an agent's action choice only
depends on information exchanged in its most recent
interaction, and not on maintaining cumulative,
time-dependent statistics (e.g., the number
of times a certain action was chosen).
This precludes adaptations of common single-agent
bandit algorithms like UCB from satisfying
Property~\ref{property:memoryless}.
With these properties in mind, 
our main algorithmic question is the following:
\begin{center}
  (Q1) \textit{Is there a memoryless
    and time-independent protocol $\calP$ that obtains
  sublinear regret?}
\end{center}

\paragraph{Best-action consensus}
Note that in addition to the standard goal
of obtaining sublinear regret, in (single-agent) online learning
and multi-armed bandit settings with stationary rewards,
a common separate objective is \textit{best-action identification}
\cite{lattimore2020bandit}. 
A successful algorithm for this task identifies 
(with some desired probability) the action
with the highest-mean reward and continues to
choose this action in future rounds in perpetuity. 

In the present multi-agent setting, this
corresponds to the distribution $\pt$
reaching (and remaining in for all subsequent rounds)
a point mass on the highest-mean action
(i.e., \textit{all} $n$ agents continue to choose this
optimal action at every round).
In the general context of the distributed
GOSSIP model, such behavior corresponds
to reaching \textit{consensus} on the
highest-mean action. Defined more formally:

\begin{definition}[Best-Action Consensus]
  \label{def:consensus}
  Under the stationary reward setting of
  Assumption~\ref{ass:stationary-proof},
  we say a protocol $\calP$ reaches
  \textit{best-action consensus} in 
  $\tau$ rounds if $\pt = (1, 0, \dots, 0)$
  for all $\tau \le t \le T$.
\end{definition}

\noindent
As mentioned, in the GOSSIP model,
many simple opinion dynamics have been proven
to quickly reach consensus (i.e., within at
most a number of rounds polynomial
in the size of the population $n$).
Thus in the present setting, in addition to
understanding whether simple opinion dynamic protocols
can obtain sublinear regret, a second natural question
to ask is:
\begin{center}
  (Q2) \textit{Is there a memoryless and
    time-independent protocol $\calP$ that reaches \\
  best-action consensus 
  within $\text{poly}(n)$ rounds?}
\end{center}

\paragraph{Related works with GOSSIP communication}
Recent works of Celis et al.~\cite{DBLP:conf/podc/CelisKV17}
and Su et al.~\cite{DBLP:journals/pomacs/SuZL19} considered
protocols in closely related settings that made
initial attempts at answering these questions.
In particular, Celis et al. proposed a simple dynamics
satisfying Property~\ref{property:memoryless} in a
slight variant of the GOSSIP model, and they considered
a stationary reward setting restricted to Bernoulli distributions. 
However, the regret bounds obtained for their protocol
cannot be made sublinear unless $n$ is exponentially
large in $T$. 
Su et al.~\cite{DBLP:journals/pomacs/SuZL19}
considered a similar dynamics, again in the stationary
Bernoulli reward setting, and they proved that best-action consensus
is eventually reached as $T \to \infty$,
but without providing any rate.
Together, these two works have left open obtaining
positive answers to questions (Q1) and (Q2). 

\subsection{Our Contributions}
\label{sec:intro:contribs}

In this work, we introduce and analyze
families of local protocols that satisfy
Property~\ref{property:memoryless} and answer
questions (Q1) and (Q2) affirmatively.
Specifically, we restrict our focus to the
\textbf{complete communication graph},
and we show that a simple instantiation of our family
of protocols \textit{\textbf{simultaneously obtains both
sublinear regret and quickly reaches best-action consensus}}:

\begin{mdframed}
  \textbf{Main Results (Informal):}
  \it
  There is a memoryless and time-independent protocol
  such that:
  \begin{itemize}[
    topsep=0em,
    leftmargin=1em,
    itemsep=0.1em
    ]
  \item
    \textnormal{\textbf{Constant regret}
    (see Theorem~\ref{thm:stationary-regret})}:
  for any time horizon $T$,
  in the stationary reward setting,
    the protocol obtains regret at most
    $R(T) \le O(\sigma \log m + \log^3 n)$
    with high probability.
  \item
    \textnormal{\textbf{Fast best-action consensus}
      (see Theorem~\ref{thm:consensus})}:
    in the stationary reward setting,
    the protocol reaches best-action
    consensus in $O(\sqrt{n} \log n)$ rounds
    with high probability. 
  \end{itemize}
\end{mdframed}

Together, these results demonstrate the surprising
power of a simple opinion-dynamics protocol
in the cooperative bandit setting. 
For better interpretation we make the following remarks:
\begin{itemize}[
  topsep=0.5em,
  leftmargin=1.5em,
  itemsep=0em,
  ]
\item
  First, not only does our protocol obtain
  sublinear regret, but it obtains only \textit{constant}
  regret with respect to $T$, scaling only with
  logarithmic dependencies on $m$ and $n$. 
  With respect to $T$ and $m$, this bound
  matches the optimal rate for the single-agent
  \textit{experts} setting~\cite{DBLP:journals/toc/AroraHK12},
  and it improves by at least a $O(\log T)$ factor
  over the best single-agent bandit regret bounds,
  which are of order $O(m \log T)$
  (assuming a constant gap $(\mu_1 - \mu_2)$)
  ~\cite{bubeck2012regret,lattimore2020bandit}.

  Thus when the population-level regret $R(T)$ 
  is viewed as the \textit{average} per-agent regret,
  the $\softO(1)$ bound of our protocol
  (which leverages inter-agent communication)
  improves over the regret of the best single-agent bandit
  algorithms (which do not use inter-agent communication). 
  Moreover, our bound holds for \textit{any}
  family of distributions satisfying the boundedness
  conditions of Assumption~\ref{reward:stationary},
  which extends beyond the Bernoulli reward setting
  considered in related work~\cite{DBLP:conf/podc/CelisKV17,
    DBLP:journals/pomacs/SuZL19}.
  
\item
  In addition to obtaining only constant regret,
  Theorem~\ref{thm:consensus} also establishes that this
  same algorithm reaches best action consensus
  within roughly $\sqrt{n}$ rounds.
  This combination of no-regret learning
  \textit{and} best-action consensus achieved by
  our protocol exhibits the type of \textit{best-of-both-worlds}
  behavior that is generally sought after in 
  other multi-agent learning settings~\cite{hu2024best}.
\end{itemize}

Moreover, using similar techniques
as in the proof of Theorem~\ref{thm:stationary-regret},
we show the additional surprising result that
our protocol is also robust to \textit{adversarial rewards}.
In this more challenging reward setting,
the distributions $\nu^t_j$ may now vary
over time (we state this formally in
Section~\ref{sec:adversarial-regret-details}).
In particular, we obtain a sublinear regret bound
for our protocol under these adversarial rewards,
so long as the time horizon $T$
does not grow too large with respect to $n$:

\begin{mdframed}
  \textnormal{\textbf{Sublinear regret for
      adversarial rewards}
    (see Theorem~\ref{thm:regret-adversarial})}:
  there is a memoryless and time-independent
  protocol such that, in the adversarial reward setting,
  when $T \le (0.5 - \eps) \log_5 n$
  for any $\eps \in (0, 0.5)$,
  the protocol obtains regret
  $R(T) \le O(\sqrt{T \log m})$. 
\end{mdframed}
Thus even in a setting with adversarially changing
reward distributions, our simple opinion-dynamics
protocol still exhibits no-regret learning behavior.
We again make several additional remarks:
\begin{itemize}[
  topsep=0.5em,
  leftmargin=1.5em,
  itemsep=0em,
  ]
\item
  First, note that from a single-agent perspective
  with adversarial rewards, the sublinear
  $O(\sqrt{T \log m})$ rate is optimal
  (with respect to both $T$ and $m$) for the
  \textit{experts} setting~\cite{DBLP:journals/toc/AroraHK12},
  and it improves  by a $\sqrt{m}$ factor
  over the optimal $O(\sqrt{T m \log m})$ regret bounds
  in the bandit setting
  \cite{bubeck2012regret}.
  Thus similarly to the case of stationary rewards,
  our protocol leads to better population-level 
  regret guarantees over the best single-agent
  method in this cooperative setting.
  Moreover, our bound in this setting also answers an
  open question of Celis et al.~\cite{DBLP:conf/podc/CelisKV17},
  who asked whether this type of simple dynamics
  could obtain sublinear regret guarantees under
  time-varying reward distributions.
\item
  On the other hand, this result restricts
  the time horizon $T$ to grow only logarithmically in $n$.
  The main barrier in allowing for larger $T$
  is related to the fact that, in this setting, the
  mass of the best action in each distribution $\pt$ may 
  still be relatively small.
  Thus for large $T$, the noise from the random agent
  communications may outweigh any meaningful signal about
  good action choices.

  In particular, in Section~\ref{sec:adversarial-regret-details},
  we also prove for our protocol a lower bound on $R(T)$
  that scales like $\Omega(T)$ once $T$ is roughly 
  $\widetilde \Omega(\sqrt{n})$.
  Thus in general, our work leaves open the
  question of whether there exist other simple protocols
  (i.e., satisfying Property~\ref{property:memoryless})
  that can obtain sublinear regret under adversarial
  rewards for longer time horizons.
\end{itemize}

\paragraph{Summary of techniques}

Our main technical insight is to show that,
for several families of very simple protocols
satisfying the memoryless and time-indepenence
attributes of Property~\ref{property:memoryless},
the evolution of the resulting distributions
$\{\pt\}$ is closely related to a certain family
of (centralized) multiplicative weights update (MWU)
processes \cite{DBLP:journals/toc/AroraHK12}.
Using this connection allows us to develop a
general analysis framework where the
regret of $\{\pt\}$ is controlled by bounding
the regret of the corresponding MWU process.
At a high-level, this approach is similar in
spirit to the analysis of~\cite{DBLP:conf/podc/CelisKV17},
which can be viewed as a special case of our framework.

\vspace*{-0.5em}
\paragraph{Structure of paper}
The remainder of the paper is structured as follows:
in Section~\ref{sec:intro:related}, we briefly summarize
several lines of related work.
In Section~\ref{sec:tech-overview}, we then
present a more detailed technical overview
of our protocols, analysis framework, and main results, and in
Section~\ref{sec:discussion} we mention several open
questions. We defer most proofs to the subsequent sections. 

\subsection{Other Related Work}
\label{sec:intro:related}

\paragraph{Multi-armed bandits and online learning}
In the single-agent case, algorithms for
the multi-armed bandit and experts settings have
been studied extensively over the past several
decades, and for background, we refer the reader to
several introductory texts
\cite{cesa2006prediction, bubeck2012regret,
  lattimore2020bandit}.
Note that the cooperative bandit
setting of the present work is distinct
from other multi-player bandit settings
where agents seek to avoid choosing the
same action in a given round
\cite{boursier2019sic,
  bubeck2020coordination}.

\paragraph{Opinion and consensus dynamics}
Simple opinion dynamics for consensus, majority,
and other algorithmic tasks have been extensively
studied in the GOSSIP model and its common variants.
We refer the interested reader
to the survey of
Becchetti et al.~\cite{becchetti2020consensus}.
These types of protocols are more generally related
to other simple time-independent and memoryless
dynamics for averaging~\cite{10.1214/11-PS184,
  10.3150/12-BEJSP04} and
rumor-spreading~\cite{chierichetti2010rumour,
  chierichetti2010almost,
  fountoulakis2012ultra,
  giakkoupis2012rumor,
  giakkoupis2014tight}.

\paragraph{Global dynamics from local protocols}
Finally, we remark that our results and analysis
adds to a growing body of work that
analyze local protocols in distributed models
by relating their global evolution to 
centralized optimization algorithms
\cite{DBLP:conf/icalp/Mallmann-TrennM18,
  DBLP:conf/nips/EvenBBFHGMT21,
  panageas2016evolutionary,
  panageas2016mixing}.

% techincal overview

\section{Technical Overview}
\label{sec:tech-overview}

\paragraph{Notation and Other Preliminaries}
Throughout, we deal with multiple sequences of vectors
indexed over rounds $t \in [T]$, for which we use the short
hand notation $\{\p^t\} := \{\p^0, \p^1, \dots, \p^t\}$.
We often compute expectation (resp., probabilities)
conditioned on two sequences $\{\pt\}$ and $\{\gt\}$
simultaneously, and we denote this double conditioning
by $\E_{t}[\cdot]$.
When we wish to condition just on a single vector $\pt$,
we will write $\E_{\pt}[ \cdot ]$.
Given $\p = (p_1, \dots, p_m)$,
we write $\E[\p]$ to denote the vector $(\E[p_1], \dots, \E[p_m])$.
We say \textit{with high probability} (w.h.p.) to mean
with probability at least $1- n^{-\Omega(1)}$
and we say \textit{n sufficiently large}
when we require $n$ larger than some absolute constant.
We use the notation $\softO(\cdot)$ to hide 
logarithmic dependencies.

\subsection{Adoption and Comparison Protocols}
\label{sec:tech-overview:protocols}

We begin by introducing two families of very simple
local protocols for this problem setting:

\begin{restatable}[\textbf{Adoption Protocols}]{protocolfam}{adoptionprotocol}
  \label{family:adoption}
  Fix a non-decreasing \textit{adoption function}
  $f: \R \to [0, 1]$. \\
  Then running an adoption protocol with function $f$
  means for each agent $u \in[n]$ the following:
  \begin{enumerate}[
    label={(\roman*)},
    topsep=0em,
    itemsep=-0.5em,
    leftmargin=2em]
  \item
    \textit{At round $t$, assume}:
    agent $u$ chose action $j$;  
    $u$ sampled agent $v$;
    and $v$ chose action $k \in [m]$.
  \item
    \textit{Then at round $t+1$}:
    agent $u$ chooses action $k$ with probability $f(g^t_k)$
    and action $j$ otherwise.
  \end{enumerate}
\end{restatable}

\begin{restatable}[\textbf{Comparison Protocols}]{protocolfam}{comparisonprotocol}
  \label{family:comparison}
  Fix a non-decreasing \textit{score function}
  $h: \R \to \R_{\ge 0}$. \\
  Then running a comparison protocol with function $h$
  means for each agent $u \in[n]$ the following:
  \begin{enumerate}[
    label={(\roman*)},
    topsep=0em,
    itemsep=-0.5em,
    leftmargin=2em]
  \item
    \textit{At round $t$, assume}:
    agent $u$ chose action $j$;  
    $u$ sampled agent $v$;
    and $v$ chose action $k \in [m]$.
  \item
    \textit{Then at round $t+1$}:
    defining
    $\rho_j \propto h\big(g^t_j\big)$
    and
    $\rho_k \propto h\big(g^t_k\big)$,
    agent $u$ chooses action $k$ with
    probability $\rho_k$ and action $j$
    with probability $\rho_j$. 
  \end{enumerate}
\end{restatable}

\noindent
In both families of protocols,
each agent's action choice at round $t+1$ is
either its own action choice at round $t$,
or that of its randomly sampled neighbor.
Thus both families of protocols satisfy
the memorlyess and time-independence attributes
of Property~\ref{property:memoryless}.
Moreover, in addition to the connection with opinion
dynamics, both adoption and comparison protocols 
are more generally related to frameworks from
social learning~\cite{ellison1995word,
  krafft2016human}, pairwise comparison
processes from evolutionary game theory
\cite{sandholm2010population,
  allen2014games, schmid2019evolutionary},
and algorithmic concepts like the
power-of-two-choices~\cite{mitzenmacher2001power}.

For concreteness, in this work we focus mainly on
the simplest instantiation of an adoption protocol,
$\betaadopt$, which uses the adoption function
$f_\beta(g) := \beta g$.
With this instantiation, an agent $u$ at round $t+1$
chooses the previous action of its neighbor with probabililty
$\beta g^t_j$, and chooses its own previous action
with probability $1 - \beta g^t_j$.
Defined formally:

\begin{restatable}{protocol}{betaadoptprotocol}
  \label{protocol:beta-adopt}
  Let $\betaadopt$ be the
  adoption protocol using the
  function
  $f_{\beta}(g) := \beta \cdot g$
  for $\beta \in (0, \tfrac{1}{\sigma}]$.
\end{restatable}

\paragraph{Evolution in Conditional Expectation}
Regardless of the function $f$,
we show that \textit{every} instantiation of an adoption
protocol is very closely related to
a certain (global) multiplicative weights update
process. To make this precise, consider
the sequence of distributions $\{\pt\}$ induced
when each agent runs an adoption protocol with
reward sequence $\{\gt\}$.
Then in conditional expectation, the coordinates
of $\p^{t+1}$ (corresponding to the fraction
of agents choosing each action at the next round)
evolve multiplicatively as follows:

\begin{restatable}{proposition}{adoptcondexp}
  \label{prop:adopt-condexp}
  Consider running an adoption protocol with
  function $f$.
  Let $f(\gt)$ denote the
  coordinate-wise application of $f$ on $\gt$.
  Then 
  $
  \E_t[p^{t+1}_j]
  =
  p^t_j
  \big(
  1 + f(g^t_j) - \langle \pt, f(\gt)\rangle
  \big)
  $
  for every $t$ and $j$.
\end{restatable}

The proof of the proposition is given in
Section~\ref{sec:protocol-details}.
In words, this result shows that in conditional expectation,
the coordinates of $\p^{t+1}$ generated via an adoption
protocol with function $f$ grow multiplicatively,
with weights depending on the strength
of $f(\g^t_j)$ relative to the global average
$\langle \p^t, f(\gt) \rangle$.
In Proposition~\ref{prop:cmp-condexp} of
Section~\ref{sec:protocol-details},
we also show that the coordinate-wise updates to $\p^{t+1}$
generated via a comparison protocol have a similar
multiplicative weights update structure
(but now depending on the values $\rho_j$ and $\rho_k$). 

Together, Propositions~\ref{prop:adopt-condexp}
and~\ref{prop:cmp-condexp} demonstrate that
\textit{every} adoption and comparison protocol follows
a more general \textit{zero-sum} multiplicative
weights update structure in conditional expectation. 
More specifically, the 
coordinates of $\hatp^{t+1} := \E_t[\p^{t+1}]$
evolve by
\begin{equation}
  \widehat p^{t+1}_j
  :=
  \E_t[p^{t+1}_j]
  \;=\;
  p^t_j \cdot \big(
  1 + F_j(\pt, \gt)
  \big) \;,
  \label{eq:condexp-zsmwu}
\end{equation}
where the set of $m$ functions $\{F_j\}^m_{j=1}$
satisfies the \textit{zero-sum} condition
$\sum_{j \in [m]} p_j \cdot F_j(\p, \g) = 0$
for all $\p \in \Delta_m$ and $\g \in \R^m$.
For example, for adoption protocols with function $f$,
Proposition~\ref{prop:adopt-condexp} shows that each
$F_j(\p, \g) = f(g_j) - \langle \p, f(\g)\rangle$.

\paragraph{Zero-Sum Multiplicative Weights Update}
The form of expression~\eqref{eq:condexp-zsmwu}
is the key structural property that we use to
analyze the regret bounds and consensus guarantees
of our adoption and comparison protocols,
and it more generally demonstrates a deeper
connection between the global evolution of
these local protocols and multiplicative weights
update procesess.

However, while the coordinates of the distribution
$\E_t[\p^{t+1}]$ are updated multiplicatively
as in~\eqref{eq:condexp-zsmwu}, these updates
are applied to the \textit{realized} distribution $\pt$,
and thus the sequence  $\{\E_t[\p^{t+1}]\}$ may not
be \textit{smooth}. 
Nevertheless, to analyze and bound the regret
of the sequence $\{\pt\}$, we can consider an idealized
process $\{\qt\}$ whose coordinates \textit{do}
evolve by composing this update structure in
each round using the same family $\{\calF\}$.
We define these \textit{zero-sum multiplicative weights update}
(zero-sum MWU) processes as follows:
% -------------------------------

\begin{definition}[Zero-Sum MWU]
  \label{def:zsmwu}
  Let $\calF = \{F_j\}$ be
  a family of $m$ function
  $F_j : \Delta_m \times \R^m \to [-1, 1]$ satisfying
  the \textit{zero-sum condition}
  $\sum_{j \in [m]} \; q_j \cdot F_j(\q, \g) = 0$
  for all $\q \in \Delta_m$ and $\g \in \R^m$.
  Then we say $\{\qt\}$ is a 
  \textit{zero-sum MWU process}
  with reward sequence $\{\gt\}$
  if for all $t \in [T]$ and $j \in [m]$:
  $
  q^{t+1}_j
  \;=\;
    q^t_j \cdot \big( 1 + F_j\big(\qt, \gt\big) \big)$.
\end{definition}
% -------------------------------

Compared to the standard (linear) versions
of multiplicative weights update
processes~\cite{DBLP:journals/toc/AroraHK12},
the zero-sum condition
$\sum_{j \in [m]} q_j \cdot F_j(\q, \g)$
ensures that each $\q^{t+1}$
always remains a distribution,
\textit{without} a renormalization step
(i.e., the simplex $\Delta_m$ is invariant to $\{\qt\}$).
In addition, zero-sum MWU processes
can be viewed as a discrete time approximation
of the continuous time \textit{replicator dynamics},
which are systems of the form
$(d q^t_j / dt) = ( u(g^t_j) - \langle \qt, u(\g^t_j)\rangle)$
for some utility function $u$,
which are very well studied in the context
of evolutionary dynamics \cite{schuster1983replicator,
  hofbauer1998evolutionary,
  cabrales2000stochastic,
  sandholm2010population}.
On the other hand, \textit{discrete-time} processes
with these zero-sum update factors are not
as well-studied.

\subsection{Analysis Framework for Bounding $R(T)$}
\label{sec:tech-overview:framework}

In order to analyze the sequences $\{\pt\}$
generated by our families of protocols,
we leverage the connection between the conditionally
expected updates to $\{\pt\}$ and zero-sum MWU processes.
In particular, we construct a general analysis
framework that couples the sequence $\{\pt\}$
with a zero-sum MWU process $\{\qt\}$ that
starts at the same initial distribution, uses the
same family $\calF$, and runs on the same reward sequence.
This framework is leveraged to prove our main results,
and we proceed to introduce its components in more detail:

% -------------------------------
% COUPLED TRAJECTIES DEF
\begin{restatable}[Coupled Trajectories]{definition}{defcoupling}
  \label{def:coupling}
  Let $\calF = \{F_j\}$ be a family of zero-sum
  functions as in Definition~\ref{def:zsmwu}.
  Given a reward sequence $\{\gt\}$,
  consider the sequences of distributions
  $\{\pt\}$, $\{\hatpt\}$, and $\{\qt\}$, each initialized
  at $\p^0  \in \Delta_m$, such that
  for all $j \in [m]$:
  \begin{align}
    q^{t+1}_j
    &\;:=\;
      q^t_j \cdot 
      \big(1 + F_j(\qt, \gt)\big) \;,
      \label{qt-def}   \\
    \widehat p^{t+1}_j
    &\;:=\;
      p^t_j \cdot
      \big(1 + F_j(\pt, \gt) \big) \;,
      \label{hatpt-def}
  \end{align}
  and where $\p^{t+1}$ is the empirical average of
  $n$ i.i.d. samples from the distribution $\hatp^t$. 
\end{restatable}

Given this coupling definition, a straightforward
calculation shows that we can over-approximate the
regret $R(T)$ of the sequence $\{\pt\}$
by (i) the regret of the sequence $\{\qt\}$ and
(ii) an additional error term incurred by this coupling.
More formally, given the setting of
Definition~\ref{def:coupling}, we have:
\begin{proposition}
  \label{prop:regret-decompose}
  $
  R(T)
    \le
    \widehat R(T)
    :=
    \max_{j \in [m]}
    \sum_{t\in [T]}
    \mu^t_j - \E[\langle \qt, \gt \rangle]
    +
    \sigma \sum_{t \in [T]}
    \E\|\qt - \pt\|_1 
  $.
\end{proposition}

\noindent
Given this proposition, our approach for bounding
$R(T)$ is to derive bounds for the
regret of the process $\{\qt\}$
and for the error terms separately. We describe
these approaches as follows:

\paragraph{Regret Bounds for Zero-Sum MWU Processes}
To bound the regret of a zero-sum MWU process $\{\qt\}$,
the key step is to relate the function value
$F_j(\q, \g)$ in conditional expectation to the quantity
$\mu^t_j - \langle \q, \bfmu^t \rangle$.
To this end we make the following assumptions on $\calF$:

% -------------------------------
% F-params assumption
\begin{restatable}{assumption}{fparams}
  \label{ass:F-params}
  Let $\calF = \{F_j\}$ be a family of functions
  as in Definition~\ref{def:zsmwu},
  and let $\{\gt\}$ be any reward sequence.
  Assume there exists
  $\alpha \in (0, 1/2]$,
  $\delta \in [0, 1]$, and $L > 0$ such that
  \begin{enumerate}[
    label={(\roman*)},
    topsep=0.5em,
    leftmargin=2.5em,
    itemsep=0em,
    ]
  \item
    for all $\p \in \Delta_m$ and $j \in [m]$:\;
    $
    \frac{\alpha}{2}
    \big(\mu^t_j - \langle \q, \bfmu^t \rangle - \delta\big)
    \;\le\;
    \E_\q[F_j(\q, \gt)]
    \;\le\;
    \frac{\alpha}{2}
    \big(\mu^t_j - \langle \q, \bfmu^t \rangle + \delta\big)
    $
  \item
    for all $\p, \q \in [m]$:
    $
    |F_j(\q, \gt) - F_j(\p, \gt)| \le L \cdot \|\p - \q\|_1
     $ .
   \end{enumerate}
\end{restatable}
% -------------------------------

Roughly speaking, condition (i) of the assumption
ensures that in (conditional) expectation,
each $F_j(\q, \g)$ is multiplicatively correlated
with the regret of action $j$ at round $t$,
up to some additive slack $\delta$,
and condition (ii) ensures that each
$F_j$ is $L$-Lipschitz. We emphasize that many natural
instantiations of the adoption and comparison protocols
from Section~\ref{sec:tech-overview:protocols}
satisfy the conditions of the assumption. 
For example, for the $\betaadopt$ protocol,
we show the following
(proved in Section~\ref{sec:protocol-details}):

\begin{restatable}{proposition}{betaadoptparams}
  \label{prop:beta-adopt-params}
  The family $\calF$ induced by the $\betaadopt$ protocol
  satisfies Assumption~\ref{ass:F-params}
  with 
  $\alpha := 2\beta$, 
  $\delta := 0$,
  and $L := 2$,
  for any $0 < \beta \le \min(1/4, 1/\sigma)$.
\end{restatable}

Under Assumption~\ref{ass:F-params},
we prove (in Section~\ref{sec:zsmwu-regret-proof})
the following bound on the expected 
regret of a zero-sum MWU process,
which is parameterized with respect to the
constants $\alpha$ and $\delta$:

\begin{restatable}[Regret of Zero-sum MWU]
  {theorem}{zsmwuregret} 
  \label{thm:zsmwu-regret}
  Consider a process $\{\qt\}$
  from Definition~\ref{def:zsmwu}
  with reward sequence $\{\gt\}$
  and using a family $\calF$ that satisfies
  Assumption~\ref{ass:F-params}
  with parameters $\alpha$ and $\delta$.
  Assume that $q^0_j \ge \rho > 0$.
  Let $\widetilde R(T) := \max_{j \in [m]}
  \sum_{t \in [T]} \mu_j - \E[\langle\qt,\gt\rangle]$.
  Then:
  \begin{itemize}[
    label=-,
    topsep=0.5em,
    itemsep=0.1em,
    leftmargin=1em
    ]
  \item
    In the stationary reward setting
    and assuming $\delta = 0$:
    $\widetilde R(T)
    \;\le\;
    \frac{4}{\alpha} \cdot \log (1/\rho)
    $.
  \item
    Moreover, in the adversarial reward setting: 
    $\widetilde R(T)
    \;\le\;
    \frac{2}{\alpha} \cdot \log (1/\rho)
    + 2 (\alpha + \delta) T$.
  \end{itemize}
\end{restatable}

To interpret these bounds,
suppose $q^0 = \1/m$ and that $\delta = 0$.
Then with stationary rewards,
the regret of $\{\qt\}$ is at most
$O(\log m)$ (i.e., constant with respect to $T$).
Additionally, for general adversarial rewards,
the regret of $\{\qt\}$ grows like $O((\log m)/
\alpha + \alpha T)$.
If $\alpha$ can be tuned to $\Theta(\sqrt{(\log m)/ T})$,
then this regret again scales sublinearly like
$O(\sqrt{T \log m})$.

Thus Theorem~\ref{thm:zsmwu-regret} shows these
zero-sum MWU processes obtain regret that match
the optimal rates (in both reward settings) for
standard MWU processes (cf., \cite{cesa2006prediction,
  DBLP:journals/toc/AroraHK12,
  lattimore2020bandit}).

\paragraph{Bounds on Coupling Error}

The second step is to
bound the error terms
$\sum_{t \in [T]} \E\|\pt - \qt\|_1$.
For this, we leverage the fact that
each $\E_t[\pt]$ and $\qt$
are both defined with the same family $\calF$
and depend similarly on $\g^{t-1}$.
Then using condition (ii) of Assumption~\ref{ass:F-params}
along with standard concentration bounds,
we derive an overall bound on
$\sum_{t\in[T]} \E\|\pt - \qt\|_1$, which we state
as Lemma~\ref{lem:coupling-error} in
Section~\ref{sec:coupling-error-details}.

\paragraph{General $\tau$-step Regret Bounds}
By the decomposition of $R(T)$
from Proposition~\ref{prop:regret-decompose},
we can then combine the zero-sum MWU regret
bound of Theorem~\ref{thm:zsmwu-regret} and
the coupling error bound of Lemma~\ref{lem:coupling-error} 
to obtain the following,
$\tau$-step regret guarantee
for the process $\{\pt\}$:

\begin{restatable}{proposition}{taustep}
  \label{prop:tau-step-regret}
  Let $R(\tau)$ be the regret of the sequence
  $\{\pt\}$ induced by a local protocol
  $\calP$ that satisfies 
  Assumption~\ref{ass:F-params}
  with parameters $\alpha$, $\delta$ and $L$.
  Let $\kappa := 3+L$, fix $c \ge 1$,
  and consider $n$ sufficiently large.
  Let $\widetilde R(\tau)$
  be the regret of the corresponding
  zero-sum MWU process $\{\qt\}$.
  Then
  $
  R(\tau)
  \;\le\;
  \widetilde R(\tau)
  +
  O\big(
  \sigma \kappa^\tau
  \cdot \sqrt{\frac{mc \log n}{n}}
  +\frac{\sigma\tau}{n^c}
  \big)
  $.
\end{restatable}

Equipped with this general analysis framework,
we can now formally present our main results
on the regret guarantees and consensus behavior
of our protocols.

\subsection{Main Result:
  Regret Bounds and Consensus for Stationary Rewards}
\label{sec:tech-overview:regret}

The coupling introduced in the previous section
provides a general framework for
analyzing the regret of any adoption or comparison protocol.
For concreteness and simplicity,
we focus on the guarantees of the $\betaadopt$ protocol
introduced in Section~\ref{sec:tech-overview:protocols}.
We prove this adoption protocol obtains only
\textit{constant regret} (with respect to $T$) 
in the stationary reward setting,
which provides an affirmative answer to first
main algorithmic question (Q1) from
Section~\ref{sec:intro:setting}.

\begin{restatable}{theorem}{stationaryregret}
  \label{thm:stationary-regret}
  Let $R(T)$ be the regret of the $\betaadopt$
  protocol in the stationary reward setting
  with $\beta := \min(1/4, 1/\sigma)$.
  Then assuming $m = o(\log^{1.5} n)$
  and $n$ sufficiently large,
  it holds with high probability that
  $R(T) \le O\big(\sigma \log m + \log^3 n \big)$.
\end{restatable}

\paragraph{Proof sketch of
  Theorem~\ref{thm:stationary-regret}:}
As mentioned, we use
the general $\tau$-step regret bound
of the analysis framework
from Section~\ref{sec:tech-overview:framework}
in conjunction with several additional arguments
depending on the magnitude of the time horizon $T$.
Specifically, 
we decompose the regret $R(T)$
into three distinct phases.
For this, we let
$\Tone :=\sqrt{n} \cdot \frac{\log m}{\sqrt{m \log n}}$
and $\Ttwo := 2 \sqrt{n} \cdot \log n$,
and we moreover define:
\begin{align*}
  \Rone
  &\;:=\;
    {\textstyle
    \; \sum_{t=1}^{\min(\Tone, T)}\;
    \mu_1 -\E[\langle \pt, \gt\rangle]
    } \;, \\
  \Rtwo
  &\;:=\;
    {\textstyle
    \; \sum_{t=\Tone}^{\min(\Ttwo, T)}\;
    \mu_1 - \E[\langle \pt, \gt\rangle]
    }
    \;\;\;\text{if $T \ge \Tone$, and 0 otherwise,} \\
  \Rthree
  &\;:=\;
    {\textstyle
    \; \sum_{t=\Ttwo}^{T}\;
    \mu_1 - \E[\langle \pt, \gt\rangle]
    }
    \;\;\qquad
    \text{if $T \ge \Ttwo$, and 0 otherwise}.
\end{align*}

Using this decomposition, we can write
$R(T) = \Rone + \Rtwo + \Rthree$,
and we derive bounds 
on each of these quantities individually:

\begin{itemize}[
  itemsep=0pt,
  leftmargin=1.5em,
  ]
\item
  \textbf{Bound on $\Rone$}:
  to start, in the first $\Tone$ rounds, we repeatedly
  apply the $\tau$-step regret bound of
  Proposition~\ref{prop:tau-step-regret}
  via an epoch-based approach similar to that
  of Celis et al.~\cite{DBLP:conf/podc/CelisKV17}.
  However, the core tool in obtaining
  a sharp bound in this first phase is to leverage
  the structure of the sequence $\{\pt\}$
  generated by the $\betaadopt$ protocol.
  In particular, we derive a lower bound on the growth
  of $p^t_1$ over time, and this allows us to
  show that the regret in each of the
  $\tau$-round epochs is sufficiently small.
  Together, this leads to the bound
  $\Rone \le O(\sigma \log m)$.
\item
  \textbf{Bounds on $\Rtwo$ and $\Rthree$}:
  For $T \ge \Tone$, we require additional
  techniques to bound the quantities
  $\Rtwo$ and $\Rthree$.
  For this, our key technical insight is to show
  that, with high probability,
  within at most $\Tone$ rounds, the mass $p^t_1$
  is at least $1 - O((\log^2 n)/\sqrt{n})$
  for every remaining round.
  By a straightforward calculation, 
  each subsequent per-round regret
  is then at most $O((\log^2 n)/\sqrt{n})$. 
  
  We then resort to using bounds on expected hitting times
  for supermartingale processes~\cite{lengler2020drift} to further show
  that, with high probability, 
  the mass $p^t_1$ reaches and remains at 1 within
  $\Ttwo$ rounds. 
  Together, this implies that both
  $\Rtwo \le O(\log^3 n)$ and $\Rthree = 0$ with high probability.
\end{itemize}
Summing these individual bounds yields
$R(T) \le O(\sigma \log m + \log^3 n)$ w.h.p.,
which holds for any $T \ge 1$.
The structure of the three phases and corresponding
regret guarantees are summarized in
Figure~\ref{fig:regret-phases}.

\begin{figure}[htb!]
  \centering
  \includegraphics[width=0.9\linewidth]{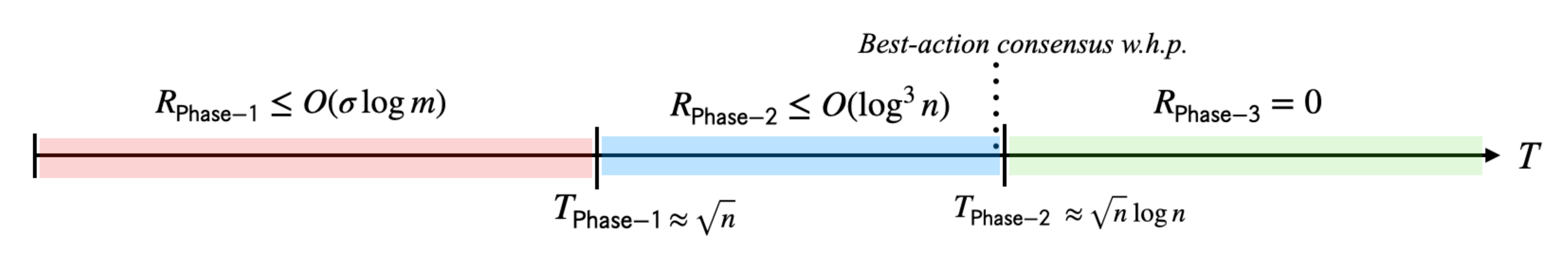}
  \caption{\small
    A depiction of the three phase structure used
    in the proof of Theorem~\ref{thm:stationary-regret},
    and the corresponding cumulative regret bound
    of each phase. The dotted vertical line denotes
    the time at which best-action consensus
    is reached with high probability
    in Theorem~\ref{thm:consensus}.}
  \vspace*{-0.5em}
  \label{fig:regret-phases}
\end{figure}

\paragraph{Reaching Best-Action Consensus}
In addition to the sublinear regret guarantee for 
$\betaadopt$, as a byproduct of the proof of
Theorem~\ref{thm:stationary-regret}
(and in particular the bound on $\Rtwo$),
we establish the following additional guarantee:
the population reaches \textit{best action consensus}
after at most $\softO(\sqrt{n})$ rounds w.h.p..

\begin{restatable}{theorem}{consensus}
  \label{thm:consensus}
  Consider the $\betaadopt$ protocol
  in the stationary reward setting of
  Theorem~\ref{thm:stationary-regret}.
  Then $p^t_1 = 1$  for all rounds
  $t \ge 2\sqrt{n} \log n$ simultaneously
  with probability at least $1 - O(1/\sqrt{n})$.
\end{restatable}

Thus Theorem~\ref{thm:consensus} answers affirmatively
the second main algorithmic question (Q2)
from Section~\ref{sec:intro:setting}.
Together with the regret bound of
Theorem~\ref{thm:stationary-regret}, these
two results demonstrate the surprising ability
of a simple opinion-dynamic-like protocol to globally
solve the online learning task of the present setting.
The full proofs of Theorem~\ref{thm:stationary-regret}
and Theorem~\ref{thm:consensus} are developed
in Section~\ref{sec:stationary-regret-proof},
and while these results are stated specifically
for the $\betaadopt$ protocol, we suspect other
instantiations of the adoption and comparison
protocol families from
Section~\ref{sec:tech-overview:protocols} yield
similar guarantees. 
  
\subsection{Robustness to Adversarial Rewards}

As a consequence of the connection between
the evolution of our local adoption and comparison
protocols and zero-sum MWU processes, we additionally
show that these protocols obtain sublinear regret
in an \textit{adversarial} reward setting
(formally defined in
Section~\ref{sec:adversarial-regret-details}). 
Specifically, for the $\betaadopt$ protocol,
we obtain the following regret
bound scaling like $\softO(\sqrt{T})$:
\begin{restatable}{theorem}{regretadversarial}
  \label{thm:regret-adversarial}
  Let $R(T)$ be the regret of the $\betaadopt$
  protocol running in a general adversarial
  reward setting with $\beta := \sqrt{(\log m)/T}$.
  Let $T \le (0.5 - \eps) \log_5 n$
  for any $\eps \in (0, 0.5)$,
  and assume $m= o(\log^{1.5} n)$. 
  Then for $n$ sufficiently large:
  $R(T) \le O(\sqrt{T \log m})$. 
\end{restatable}

The proof of the theorem again uses the
bound of Proposition~\ref{prop:tau-step-regret}
from the general analysis framework,
and we present the details in 
Section~\ref{sec:adversarial-regret-details}.
While the result shows the additional ability 
of the $\betaadopt$ protocol to obtain sublinear
regret in this more challenging reward setting,
the theorem restricts the time horizon $T$ to grow
only logarithmically in $n$.
We further prove in Proposition~\ref{prop:adversarial-lower}
that for the $\betaadopt$ protocol,
there exist adversarial reward sequences
that force $R(T) \ge \Omega(T)$ once
$T \ge \widetilde \Omega(\sqrt{n})$.
We leave the study of whether other simple protocols
can obtain sublinear regret for longer time horizons
in the adversarial reward setting as future work.

% discussion

\section{Discussion and Open Problems}
\label{sec:discussion}

In this work, we proved that a simple opinion-dynamic-like 
protocol is a no-regret learning algorithm
in the cooperative bandit setting.
Moreover, we showed this protocol simultaneously reaches
best-action consensus in the stationary reward setting, 
and also obtains sublinear regret guarantees in the
adversarial reward setting.
Our results leave open several lines of questions:
\begin{itemize}[
  leftmargin=1.5em,
  topsep=0.5em,
  itemsep=0em
]
\item
  First, notice that Theorems~\ref{thm:stationary-regret},
  \ref{thm:consensus}, and~\ref{thm:regret-adversarial}
  all assume the size of the action set $m$ grows
  no more than polylogarithmically in $n$.
  Establishing whether this constraint
  can be loosened is left for future work.
 
\item
  In this work, we considered the complete
  communication graph as an initial case study
  for understanding the regret and consensus guarantees
  of our simple families of protocols.
  Thus a second immediate question is to quantify the
  regret of these simple protocols when the
  underlying communication graph is non-complete.
  
  For general, non-complete topologies,
  we would expect the regret of our algorithms to
  have some dependence on the expansion or other
  combinatorial properties of the underlying graph.
  This is a phenomenon that appears in
  other cooperative bandit works within different
  communication models~\cite{cesa2016delay,
    bar2019individual,martinez2019decentralized},
  as well as in the convergence guarantees of 
  opinion dynamics for other tasks in
  the GOSSIP model~\cite{chierichetti2010rumour,
    giakkoupis2012rumor,
    giakkoupis2014tight,
    DBLP:conf/icalp/Mallmann-TrennM18}.
\item
  Finally, it remains open to establish whether
  other memoryless and time-independent protocols
  can obtain sublinear regret guarantees
  in the adversarial reward setting beyond the
  $O(\log n)$-round time horizon of
  Theorem~\ref{thm:regret-adversarial}. 
\end{itemize}

% protocol details

\section{Details on Local Protocols}
\label{sec:protocol-details}

In this section, we provide more details on
the families of adoption and comparison protocols
introduced in Section~\ref{sec:tech-overview:protocols}.
To start, we make the following remark regarding the 
neighbor sampling mechanism of the GOSSIP model
that is assumed in the problem setting:

\begin{remark}
  \label{remark:gossip}
  Recall in this work that we consider a \textit{complete}
  communication graph. We assume further this communication
  graph contains self-loops, which is a standard assumption in
  the GOSSIP model~\cite{boyd2006randomized, shah2009gossip,
    becchetti2014plurality,becchetti2020consensus}.
  Moreover, we use the terminology ``the neighbor
  sampled by an agent $u$'' to refer to the uniformly random
  neighbor that agent $u$ exchanges information with under
  the GOSSIP-style communication of the problem setting.
  
  However, we remark these random information exchanges should
  be viewed as being ``scheduled'' by the model. In other words,
  in the GOSSIP model, the agent $u$ does not
  explicitly perform the neighbor sampling itself, 
  but rather the model stipulates that in each round,
  every agent $u$ has a (one-sided) information exchange
  with a uniformly random neighbor.
\end{remark}

\subsection{Evolution of Adoption Algorithms}
\label{app:protocols:adoption}

For adoption algorithms, we
prove Proposition~\ref{prop:adopt-condexp},
which derives the structure of their
coordinate-wise updates in conditional expectation.
We restate the both the
definition of Adoption Protocols and the
proposition here for convenience:

\adoptionprotocol*

\adoptcondexp*

\begin{proof}
  First, letting $c^{t+1}_u \in [m]$ denote the action chosen by
  agent $u \in [n]$ in round $t+1$, observe that
  \begin{equation*}
    \E_t [p^{t+1}_j]
    \;=\;
    \frac{1}{n} \sum_{u \in [n]}
    \Pr_t[ c^{t+1}_u = j ] \;,
  \end{equation*}
  which follows from the fact that $p^{t+1}_j$ is the average of the
  $n$ indicators $\1\{c^{t+1}_u = j\}$.
  By the local symmetry of the algorithm and communication model,
  $\Pr_t[ c^{t+1}_u = j ]$ is identical for all agents $u$.
  However, this value is dependent on $c^t_u$ (i.e., the
  action choice of $u$ at the previous round $t$).
  
  Thus using the law of total probability,
  for any agent $u$ we can write
  \begin{align*}
    \Pr_t[c^{t+1}_u = j]
    \;=\;
    \1\{ &c^t_u = j\}
           \cdot \Pr_t [c^{t+1}_u = j | c^t_u = j]  \\
         &+ \sum_{k \neq j \in [m]}
           \1\{ c^t_u = k\} \cdot \Pr_t [c^{t+1}_u = j  | c^t_u = k] \;.  
  \end{align*}
  Now fix agent $u$, and let $v \in [n]$ denote the
  agent that $u$ samples in round $t$.
  Now recall from the definition of the algorithm that
  if $c^t_u = k \neq j$, then $c^{t+1}_u = j$ with probability $f(g^t_j)$
  only if agent $v$ chose action $j$ in round $t$, i.e., $c^t_v = j$.
  On the other hand, if $c^t_u = j$, then $c^{t+1}_u = j$ either
  if $c^t_v = j$, or if $c^t_v = k \neq j$ and agent $u$
  \textit{rejects} adopting action $k$ with probability $1 - f(g^t_k)$.

  Thus we have
  \begin{align*}
    \Pr_t[c^{t+1}_u = j | c^t_u = j]
    &\;=\;
      p^t_j + \sum_{k \neq j \in [m]} p^t_k \cdot (1 - f(g^t_k)) \\
    \text{and}\;\;\;
    \Pr_t[c^{t+1}_u = j | c^t_u = k]
    &\;=\;
      p^t_j \cdot f(g^t_j)
      \;\;\;\text{for $k \neq j$}\;.
  \end{align*}
  Combining these cases, noting also that
  $\frac{1}{n}\sum_{u \in [v]} \1\{ c^t_u = k\} = p^t_k$ for any $k \in [m]$,
  and using the fact that $\sum_{k \in [m]} p^t_k = 1$, we
  can then write
  \begin{align*}
    \E_t[p^{t+1}_j]
    &\;=\;
      p^t_j \cdot
      \big(
      p^t_j + \sum_{k \neq j \in [m]} p^t_k \cdot (1- f(g^t_k))
      \big)
      +
      \sum_{k \neq j \in [m]}
      p^t_k \cdot 
      \big(
      p^t_j \cdot f(g^t_j)
      \big) \\
    &\;=\;
      p^t_j
      \cdot
      \Big(
      1 + \sum_{k \neq j \in [m]}
      p^t_k \cdot \big(f(g^t_j) - f(g^t_k)\big)
      \Big) \\
    &\;=\;
      p^t_j
      \cdot
      \Big(
      1 + f(g^t_j) - \big\langle \pt, f(\gt) \big\rangle
      \Big) \;,
  \end{align*}
  which concludes the proof. 
\end{proof}

Importantly, we also verify that such multiplicative updates
in every coordinate $j$ still lead to a proper distribution:
for this, it is easy to check that
\begin{align*}
  \sum_{j \in [m]}
  \E_t[p^{t+1}_j]
  &\;=\;
  \sum_{j \in [m]}
  p^t_j
  \cdot
  \Big(
  1 + f(g^t_j) - \big\langle \pt, f(\gt) \big\rangle 
  \Big) \\
  &\;=\;
  \sum_{j \in [m]} p^t_j
  +
  \big\langle \pt, f(\gt)\big\rangle - \big\langle \pt, f(\gt)\big\rangle
  \;=\;
  1 \;,
\end{align*}
which holds since $\pt$ is a distribution. \\

\subsubsection{Example Instantiation:
  the $\betaadopt$ protocol}

As introduced in Section~\ref{sec:tech-overview},
the simplest instantiation of
an adoption protocol is $\betaadopt$:

\betaadoptprotocol*

Using the zero-sum MWU terminology of
Defnition~\ref{def:zsmwu} and Assumption~\ref{ass:F-params},
observe by Proposition~\ref{prop:adopt-condexp}
that running this protocol induces the
zero-sum family $\calF = \{F_j\}_{j \in [m]}$ where
$
 F_j(\p, \g)
  \;=\;
  \sum_{j \in [m]}
  p_j \cdot
  (f_{\beta}(g_j) - f_{\beta}(g_k))
$
for each $j \in [m]$. 
For $\betaadopt$, this simplifies to
\begin{equation}
 F_j(\p, \g)
 \;=\;
 \beta \cdot (g_j - \langle \p, \g \rangle)
 \label{eq:disc-adopt-Fj}
\end{equation}
for all $j \in [m]$. 
Then we can show this algorithm statisfies
Assumption~\ref{ass:F-params} as follows:

\betaadoptparams*

\begin{proof}
  By expression~\eqref{eq:disc-adopt-Fj}, and taking
  expectations conditioned on $\p$, we have
  \begin{equation*}
    \E_{\p}[F_j(\p, \g)]
    \;=\;
    \beta \mu_j - \beta\langle \p, \bfmu\rangle
    \;=\;
    \beta \cdot (\mu_j - \langle \p, \bfmu\rangle)
  \end{equation*}
  where we define $\bfmu = \E[\g]$.
  Thus $\betaadopt$ trivially satisfies
  Assumption~\ref{ass:F-params} with
  $\delta = 0$ and 
  $\alpha = 2\beta$
  for $\beta \le \min\{\frac{1}{4}, \frac{1}{\sigma}\}$.
  In particular, the constraint $\beta \le \frac{1}{4}$ ensures
  $\alpha \le \frac{1}{2}$,  
  and the constraint $\beta \le \frac{1}{\sigma}$ ensures
  $f_{\beta}(g_j) \le 1$. 
  Also, assuming
  $\beta \le \min\{\frac{1}{4}, \frac{1}{\sigma}\}$, 
  notice that $|f_{\beta}(g_j) - f_{\beta}(g_k)| \le 2$
  for any $g_j, g_k \in [0, \sigma]$, and thus
  setting $L=2$ is sufficient to satisfy
  condition (ii) of the assumption. 
\end{proof}

We remark that other natural choices of adoption functions
$f$ include the sigmoid function, but for simplicity
we focus mainly on the guarantees obtained by
the $\betaadopt$ protocol.

\subsection{Evolution of Comparison Protocols}
\label{app:protocols:comparison}

Recall the definition of comparison protocols
from Section~\ref{sec:tech-overview:protocols}:

\comparisonprotocol*

Natural choices for the score function $h$ include
an exponential function. However, regardless
of the choice of $h$, the conditionally expected
coordinate-wise updates to $\pt$ have
a multiplicative structure similar to that
of Adoption protocols. 
Specifically, we derive
the following proposition, analogous to Proposition~\ref{prop:adopt-condexp}:

\begin{restatable}{proposition}{cmpcondexp}
  \label{prop:cmp-condexp}
  Let $\{\pt\}$ be the sequence induced when every
  agent runs the comparison protocol with score function $h$
  and reward seqeunce $\{\gt\}$. 
  Furthermore, for any $\g \in \R^m$ and $j \in [m]$,
  let $H(\g, j) \in [-1, 1]^m$ be the
  $m$-dimensional vector whose $k$'th coordinate
  is given by $\rho_j - \rho_k$. 
  Then 
  $
  \E_t[ p^{t+1}_j]
  =
  p^{t}_j
  \cdot
  \big(
  1 + \langle \pt, H(\gt, j)\rangle
  \big) 
  $
  for every $t$ and $j$.
\end{restatable}

\begin{proof}
  Fix $j \in [m]$ and $t \in [T]$. Again let $c^t_i \in [m]$
  denote the action chosen by agent $i \in [n]$ at round $t$. 
  Then observe that we can write
  \begin{align*}
    \E_t\big[ p^{t+1}_j\big]
    \;=\;
      \frac{1}{n} \sum_{i \in [n]} \Pr_t \big[ c^{t+1}_i = j\big] 
    \;=\;
    \frac{1}{n} \sum_{i \in [n]}
    &\Bigg(
      \1\{c^{t}_i = j\} \cdot \Pr_t\big[c^{t+1}_i=j\;|\; c^{t}_i=j\big] \\
    &+ \sum_{k \neq j \in [m]}
      \1\{c^{t}_i = k\} \cdot \Pr_t\big[c^{t+1}_i=j\;|\; c^{t}_i=k\big]
      \Bigg) \;.
  \end{align*}
  In the case that $c^{t}_i = j$, note that $c^{t+1} = j$ with probability
  1 if agent $i$ samples an agent $u$ that also chose action $j$ in round $t$.
  Otherwise, if agent $u$ chose some action $k \neq j$, then agent $i$
  chooses action $j$ with probability
  $1 - h(g^{t}_k)/ \big(h(g^{t}_j) + h(g^{t}_k)\big)$.
  Together, this means that
  \begin{align}
    \Pr_t\Big[c^{t+1}_i=j\;|\; c^{t}_i=j\Big]
    &\;=\;
      p^{t}_j +
      \sum_{k \neq j \in [m]}
      p^{t}_k
      \Bigg(
      1 - \frac{h(g^{t}_k)}{h(g^{t}_j) + h(g^{t}_k)}
      \Bigg) \nonumber \\
    &\;=\;
      1 -
      \sum_{k \neq j \in [m]}
      p^{t}_k
      \cdot 
      \frac{h(g^{t}_k)}{h(g^{t}_j) + h(g^{t}_k)}
       \;\;\;.
    \label{eq:comp:1}
  \end{align}
  In the other case when $c^{t}_i = k \neq j$, then $c^{t+1} = j$
  only when agent $i$ samples a neighbor that chose action $j$ in round $t$,
  and thus
  \begin{equation}
    \Pr_t\big[c^{t+1}_i=j\;|\; c^{t}_i=k\big]
    \;=\;
    p^{t}_j \cdot \frac{h(g^{t}_j)}{h(g^{t}_j) + h(g^{t}_k)} \;.
    \label{eq:comp:2}
  \end{equation}
  Now observe that for any $k \in [m]$,
  $\tfrac{1}{n} \sum_{i \in [n]} \1\{c^{t}_i = k\} = p^{t}_k$ by definition.
  Then together with expression~\eqref{eq:comp:1} and~\eqref{eq:comp:2},
  we have
  \begin{align*}
    \E_t\big[ p^{t+1}_j\big]
    &\;=\;
      p^{t}_j\Bigg(
      1 -
      \sum_{k \neq j \in [m]}
      p^{t}_k
      \cdot 
      \frac{h(g^{t}_k)}{h(g^{t}_j) + h(g^{t}_k)}
      \Bigg)
      +
      \sum_{k \neq j \in [m]}
      p^{t}_k \cdot
      p^{t}_j \frac{h(g^{t}_j)}{h(g^{t}_j) + h(g^{t}_k)} \\
    &\;=\;
      p^{t}_j
      \cdot \Bigg[
      1 + \sum_{k \neq j \in [m]}
      p^{t}_k \cdot
      \frac{h(g^{t}_j) - h(g^{t}_k)}{h(g^{t}_j) + h(g^{t}_k)}
      \Bigg] \\
    &\;=\;
      p^{t}_j
      \cdot \Big[
      1 +  \big\langle \pt, H(\gt, j)\big\rangle
      \Big] \;,
  \end{align*}
  which concludes the proof.
\end{proof}

Again, we also verify that for any $\p \in \Delta_m$ and $\g \in [0, 1]^m$,
the family of functions $\{\langle \p, H(\g, j) \rangle\}_{j \in [m]}$
satisfies the zero-sum property
$\sum_{j \in [m]} p_j \cdot \langle \p, H(\g, j)\rangle = 0$.
To see this, observe that
\begin{align*}
  \sum_{j \in [m]}
  p_j \cdot \langle \p, H(\g, j)\rangle
  &\;=\;
    \sum_{j \in [m]}
    p_j \cdot \sum_{k \in [m]} p_k \cdot
    \frac{h(g^{t}_j) - h(g^{t}_k)}{h(g^{t}_j) + h(g^{t}_k)} \\
  &\;=\;
    \sum_{(j, k) \in [m] \times [m]}
    p_j \cdot p_k  \cdot \Bigg(
    \frac{h(g^{t}_j) - h(g^{t}_k)}{h(g^{t}_j) + h(g^{t}_k)}
    + 
    \frac{h(g^{t}_k) - h(g^{t}_j)}{h(g^{t}_j) + h(g^{t}_k)}
    \Bigg)
    \;=\;
    0 \;.
\end{align*}

% details on zero-sum MWu

\section{Details on Zero-Sum MWU Regret Bounds}
\label{sec:zsmwu-regret-proof}

In this section, we prove Theorem~\ref{thm:zsmwu-regret},
which bounds the regret of a Zero-Sum MWU process
from Definition~\ref{def:zsmwu}. The regret bound relies
on some smoothness assumptions of the zero-sum MWU
process which are summarized in Assumption~\ref{ass:F-params}.
We restate the assumptions and theorem  here:

\fparams*

\zsmwuregret*

Under the conditions of Assumption~\ref{ass:F-params},
the proof of Theorem~\ref{thm:zsmwu-regret}
leverages standard ``potential function''
approaches to proving MWU regret bounds
(i.e., in the spirit
of Arora et al.~\cite{DBLP:journals/toc/AroraHK12}),
but with some additional bookkeeping to account
for the $\alpha$ and $\delta$ parameters.
In the stationary reward setting 
and assuming $\delta = 0$, we can further use
the fact that $\mu_1 \ge \langle \pt, \bfmu \rangle$
for all rounds $t$ to derive a much
smaller cumulative regret bound
that is only a constant with respect to $T$.

\begin{proof}
  Fix $j \in [m]$ and $t \in [T]$.
  Recall that in round $t$, both $\qt$ and $\gt$
  are random variables, where $\qt$ depends on
  the randomness in both $\{\q^{t-1}\}$
  and $\{\g^{t-1}\}$.
  Then conditioning on both of these sequences
  (which is captured in the notation $\E_{t-1}[\cdot]$),
  we can use the form of the zero-sum MWU 
  update rule of Definition~\ref{def:zsmwu} to write
  \begin{align*}
    \E_{t-1}\big[ q^t_j \big]
    &\;=\;
      \E_{t-1}\big[\;
      q^t_j
      \cdot
      \big(
      1 + F_j(\qt, \gt)
      \big)
      \;\big] \\
    &\;=\;
      q^t_j
      \cdot
      \E_{t-1}\big[1 +  F_j(\qt, \gt)\big]
      \;=\;
      q^t_j
      \cdot
      \big(1 +  \E_{t-1}\big[F_j(\qt, \gt)\big]\big)
      \;.
  \end{align*}
  Here, the second equality comes from the fact that
  $\qt$ is a constant when conditioning on $\{\q^{t-1}\}$ and 
  $\{\g^{t-1}\}$, and thus $\E_{t-1}[q^t_j] = q^t_j$. 
  Now for readability, let us define
  \begin{equation*}
    m^t_j \;:=\;
    \E_{t-1}\big[ F_j(\qt, \gt) \big] \;,
  \end{equation*}
  and that $m^t_j$ is deterministic (meaning $\E[m^t_j] = m^t_j$),
  since the only remaining randomness after the conditioning
  is with respect to $\gt$. 
  Thus using the law of iterated expectation, we can ultimately write
  \begin{equation*}
    \E\big[q^{t+1}_j\big]
    \;=\;
    \E\big[\; \E_{t-1}\big[q^{t+1}_j\big] \;\big]
    \;=\;
    \E\big[ q^t_j \big] \cdot \big(1 + m^t_j \big) \;.
  \end{equation*}
  By repeating the preceding argument for each of
  $\E[q^{t-1}_j], \dots, \E[q^{1}_j]$, and setting
  $T = t+1$, we find that 
  \begin{equation}
    \E\big[q^{T}_j\big]
    \;=\;
    q^{0}_j \cdot \prod_{t \in [T-1]} \big( 1 + m^{t}_j \big) \;.
    \label{pfmwu:rhs:1}
  \end{equation} 
  From here, we roughly follow a standard multiplicative weights analysis:
  first, define the sets $M_j^+$ and $M_j^-$ as
  \begin{align*}
    M_j^+
    &\;=\;
      \{ t \in [T-1]\;:\; m^t_j \ge 0\} \\
    \text{and}\;\;
    M_j^-
    &\;=\;
      \{ t \in [T-1]\;:\; m^t_j < 0\} \;,
  \end{align*}
  where clearly $M^+_j \cup M^-_j = [T]$. 
  Then we can rewrite expression~\eqref{pfmwu:rhs:1} as
  \begin{equation*}
    \E[q^{T}_j]
    \;=\;
    q^0_j \cdot
    \prod_{t \in M^+_j} \big( 1 + m^t_j \big)
    \cdot
    \prod_{t \in M^-_j} \big( 1 + m^t_j \big) \;.
  \end{equation*}
  Now for each $t$, define
  $
  \Delta^t_j := \mu^t_j - \langle \qt, \bfmu^t \rangle
  \in [-1, 1].
  $
  Here, the bounds on the range of each $\Delta^t_j$
  follows from the assumption that all $\mu^t_j \in [0, 1]$. 
  % \end{equation*}
  Using this notation, observe that
  Assumption~\ref{ass:F-params} implies
  $m^t_j \ge \frac{\alpha}{2} \cdot (\Delta^t_j - \delta)$
  for every $t$ and $j$. 
  Observe that if $\Delta^t_j < 0$, then
  we must have $m^t_j < 0$. On the other hand,
  when $\Delta^t_j \ge 0$, the sign of $m^t_j$ is
  ambiguous.
  To distinguish the two cases, we define the additional
  two sets $G_j^+$ and $G_j^-$ as
  \begin{align*}
    G_j^+
    &\;=\;
      \{ t \in [T-1]\;:\; \Delta^t_j - \delta \ge 0\} \\
    \text{and}\;\;
    G_j^-
    &\;=\;
      \{ t \in [T-1]\;:\; \Delta^t_j - \delta < 0\} \;.
  \end{align*}  
  Combining the pieces above, it follows that
  we can write
  \begin{align*}
    \E[q^T_j]
    &\;\ge\;
      q^0_j
      \cdot
      \prod_{t \in M^+_j}
      \Big(1+ \tfrac{\alpha (\Delta^t_j - \delta)}{2}\Big)
      \cdot
      \prod_{t \in M^-_j}
      \Big(1 + \tfrac{\alpha(\Delta^t_j - \delta)}{2}\Big) \\
    &\;=\;
      q^0_j
      \cdot
      \prod_{t \in M^+_j \cap G^+_j}
      \Big(1+ \tfrac{\alpha(\Delta^t_j - \delta)}{2}\Big)
      \cdot
      \prod_{t \in M^+_j \cap G^-_j}
      \Big(1+ \tfrac{\alpha (\Delta^t_j - \delta)}{2}\Big)
      \cdot
      \prod_{t \in M^-_j}
      \Big(1 + \tfrac{\alpha(\Delta^t_j - \delta)}{2}\Big) \;.
  \end{align*}
  Observe also that each $|\Delta^t_j - \delta| \in [-2, 2]$
  by the definition of $\Delta^t_j$ and by the assumption that
  $\delta \in [0, 1]$. 
  Thus we can then use the facts that 
  $(1+\alpha x) \ge (1+\alpha)^x$ for $x \in [0, 1]$,
  and that $(1+\alpha x)\ge (1-\alpha)^{-x}$ for $x \in [-1, 0$],
  which allows us to further simplify and write
  \begin{equation*}
    \E[q^T_j]
    \;\ge\;
    q^0_j
    \cdot
    \prod_{t \in M^+_j \cap G^+_j}
    (1+\alpha)^{\frac{(\Delta^t_j - \delta)}{2}}
    \cdot
    \prod_{t \in M^+_j \cap G^-_j}
    (1-\alpha)^{-\frac{(\Delta^t_j - \delta)}{2}}
    \cdot
    \prod_{t \in M^-_j}
    (1-\alpha)^{-\frac{(\Delta^t_j- \delta)}{2}} \;.
  \end{equation*}
  Now using the fact that $q^T_j \le 1$, taking logarithms,
  and multiplying through by 2, we find 
  \begin{align}
    0
    \;\ge\;
    2 \log q^0_j
    \;+ &\sum_{t \in M^+_j \cap G^+_j}
          \log(1+\alpha) (\Delta^t_j - \delta)  \nonumber \\
    - &\sum_{t \in M^+_j \cap G^-_j}
        \log(1-\alpha) (\Delta^t_j - \delta) 
        \;- \sum_{t \in M^-_j} \log(1-\alpha) (\Delta^t_j - \delta) \;.
        \label{pfmwu:rhs:2}
  \end{align}
  From here, we conclude the proof by considering
  the stationary and adversarial reward settings separately.
  
  \paragraph{Stationary reward setting with $\delta = 0$:}
  We start with stationary reward setting
  (as specified in Assumption~\ref{reward:stationary})
  and further assume that $\delta = 0$.
  Consider $j=1$, and observe by the assumption of
  on the ordering of coordinates in $\bfmu$ that 
  $\Delta^t_1 = \mu_1 - \langle \qt, \bfmu \rangle \ge 0$ for
  all $\qt$. Thus by definition, we have $M^+_1 \cap G^+_1 = T$,
  and expression~\eqref{pfmwu:rhs:2} simplifies to
  \begin{equation*}
    0
    \;\ge\;
    2 \log q^0_j
    \;+ \sum_{t \in [T]}
    \log(1+\alpha) \Delta^t_1 \;.
  \end{equation*}
  Using the identity $\log (1+x) \ge x-x^2 \ge x/2$ (which holds
  for all $x \in (0, 1/2]$) and rearranging terms then yields
  \begin{equation}
    \sum_{t \in [T]} \Delta^t_1
    \;\le\;
    \frac{4 \log (1/q^0_j)}{\alpha}
    \;.
    \label{eq:zsmwu:stat:1}
  \end{equation}
  By definition, we have
  $\Delta^t_1 = \mu_1 - \langle \qt, \bfmu\rangle$
  for each $t$, and thus we can write
  $\langle \qt, \bfmu \rangle = \langle \qt, \E[\gt] \rangle
  = \E_{\qt}[\langle \qt, \gt\rangle]$.
  Moreover, by the assumption that $\q^0_1 \ge \rho > 0$
  (with probability 1), we have that 
  \begin{equation*}
    \sum_{t \in [T]}\mu_1 - \E_{\qt}[\langle \qt,\pt\rangle]
    \;\le\;
    \frac{4 \log (1/\rho)}{\alpha} \;.
  \end{equation*}
  Then taking expectations and using the
  law of iterated expectation, we have
  \begin{equation*}
    \widetilde R(T, 1)
    \;=\;
    \sum_{t \in [T]} \mu_1 - \E[\langle \qt,\pt\rangle]
    \;\le\;
    \frac{4 \log (1/\rho)}{\alpha}\;,
  \end{equation*}
  which proves the claim for the stationary case. 

  \paragraph{General adversarial reward setting:}
  We now consider the general adversarial reward setting
  and pick back up from expression~\eqref{pfmwu:rhs:2}. 
  By definition, recall that $(\Delta^t_j - \delta)$
  is non-negative
  for $t \in M^+_j \cap G^+_j$, and negative
  for $t \in M^+_j \cap G^-_j$ and $t \in M^-_j$.
  Thus again using the identities
  $\log (1+x) \ge x - x^2$ and $- \log (1-x) \le x + x^2$,
  which both hold for all $x \in [0, \tfrac{1}{2}]$, we can
  further bound expression~\eqref{pfmwu:rhs:2}
  and rearrange to find
  \begin{align}
    % 2 &\log (1/q^0_j) \nonumber \\
    2 \log (1/q^0_j)
      &\;\ge\;
        \sum_{t \in M^+_j \cap G^+_j}
        (\alpha-\alpha^2) (\Delta^t_j - \delta)  % \nonumber \\
        % &\quad\quad\quad
        + \sum_{t \in M^+_j \cap G^-_j}
    (\alpha + \alpha^2) (\Delta^t_j - \delta) 
    + \sum_{t \in M^-_j}
    (\alpha + \alpha^2) (\Delta^t_j - \delta)
      \nonumber \\
    &\;=\;
      \sum_{t \in M^+_j \cap G^+_j} \Delta^t_j \alpha
      \;-\sum_{t \in M^+_j\cap G^+_j} \Delta^t_j \alpha^2
      \;-\sum_{t \in M^+_j\cap G^+_j} \delta (\alpha - \alpha^2)
      \nonumber \\
    &\quad\quad\quad+
      \sum_{t \in M^+_j \cap G^-_j} \Delta^t_j \alpha
      \;+ \sum_{t \in M^+_j \cap G^-_j} \Delta^t_j \alpha^2
      \;- \sum_{t \in M^+_j \cap G^-_j} \delta (\alpha + \alpha^2)
      \nonumber \\
    &\quad\quad\quad+
      \sum_{t \in M^-_j} \Delta^t_j \alpha
      \;+ \sum_{t \in M^-_j} \Delta^t_j \alpha^2
      \;- \sum_{t \in M^-_j} \delta (\alpha + \alpha^2) \;.
      \label{pfmwu:rhs:2.5}
  \end{align}
  Now in expression~\eqref{pfmwu:rhs:2.5}, there are
  nine individual summations. We collect these into
  into four groups, and then bound and simplify as follows:
  \begin{align*}
    &\text{(i)}\;\;\;
      \sum_{t \in M^+_j \cap G^+_j}\Delta^t_j \alpha
      \;+\sum_{t \in M^+_j \cap G^-_j}\Delta^t_j \alpha
      \;+\sum_{t \in M^-_j}\Delta^t_j \alpha
      \;=\;
      \alpha \sum_{t \in [T]} \Delta^t_j \\
    &\text{(ii)}\;\;\;
      - \sum_{t \in M^+_j \cap G^+_j}\Delta^t_j \alpha^2
      \;+ \sum_{t \in M^+_j \cap G^-_j}\Delta^t_j \alpha^2
      \;+ \sum_{t \in M^-_j} \Delta^t_j \alpha^2
      \;\ge\;
      - \alpha^2 \sum_{t \in [T]} |\Delta^t_j| \\
    &\text{(iii)}\;\;\;
      - \sum_{t \in M^+_j \cap G^+_j}\delta (\alpha - \alpha^2)
      \;-\sum_{t \in M^+_j \cap G^-_j}\delta (\alpha + \alpha^2)
      \;-\sum_{t \in M^-_j} \delta (\alpha + \alpha^2)
      \;\ge\;
      - 2\alpha \sum_{t \in [T]} \delta \;.
  \end{align*}
  In the above, we use in (ii) the fact that
  $\Delta^t_j \ge - |\Delta^t_j|$ for any $t$,
  and in (iv) the facts that $\alpha - \alpha^2 \le 2\alpha$
  and $\alpha + \alpha^2 \le 2\alpha$.   
  Then substituting these groups back into
  expression~\eqref{pfmwu:rhs:2.5}, we ultimately find that
  \begin{align}
    2 \log (1/q^0_j)
    &\;\ge\;
      \alpha \sum_{t \in [T]} \Delta^t_j
      - \Big(
      (\alpha^2 +  \sum_{t \in [T]}  |\Delta^t_j|
      + 2\alpha \sum_{t \in [T]} \delta
      \Big) \nonumber \\
    &\;\ge\;
      \alpha \sum_{t \in [T]} \Delta^t_j
      - \big(
      (\alpha^2 + \delta \alpha \big) \cdot 2T
      \;, \nonumber 
  \end{align}
  where the final inequality comes from the fact
  that $|\Delta^t_j| \le 2$ and the assumption that
  $\delta \le 1$.
  Thus using the definition
  $\Delta^t_j = \mu^t_j - \langle \qt,\bfmu^t\rangle$,
  we can rearrange to write
  \begin{equation*}
    \sum_{t \in [T]}
    \mu^t_j - \langle \qt, \bfmu^t \rangle
    \;\le\;
    \frac{4 \log (1/q^0_j)}{\alpha} +
    2 (\alpha + \delta)\cdot  T \;.    
  \end{equation*}
  Finally, as in the stationary case, observe for each
  $t \in [T]$ that 
  $
  \langle \qt, \bfmu^t \rangle
  =
  \langle \qt, \E[\gt] \rangle
  =
  \E_{\qt}[\langle \qt, \gt \rangle] \;,
  $
  and recall also by assumption that $q^0_j \ge \rho > 0$
  with probability 1. 
  Thus with this same probability:
  \begin{equation*}
    \sum_{t \in [T]}
    \mu^t_j - \E_{\qt}[\langle \qt, \gt \rangle]
    \;\le\;
    \frac{2 \log (1/\rho)}{\alpha}
    +
    2 \big(\alpha + \delta\big) \cdot T
    \;.
  \end{equation*}
  Then taking expectations on both sides and using
  the law of iterated expectation, we find
  \begin{equation*}
    \sum_{t \in [T]}
    \mu^t_j - \E[\langle \qt, \gt \rangle]
    \;\le\;
    \frac{2 \log (1/\rho)}{\alpha}
    +
    2 \big(\alpha + \delta\big) \cdot T \;,
  \end{equation*}
  which concludes the proof for the general,
  adversarial reward setting. 
\end{proof}

% coupling error proofs

\section{Details on Coupling Error Bounds}
\label{sec:coupling-error-details}

In this section, we develop the
proof of the following lemma,
which bounds the error on the coupling
from Definition~\ref{def:coupling}:

\begin{restatable}{lemma}{couplingerror}
  \label{lem:coupling-error}
  Consider the
  sequences $\{\pt\}$, $\{\hatpt\}$,
  and $\{\qt\}$ from
  Definition~\ref{def:coupling}
  with a reward sequence $\{\gt\}$
  and using a family $\calF$ that satisfies
  Assumption~\ref{ass:F-params} with
  parameter $L$. Let $\kappa := (3+L)$,
  and assume $n \ge 3c \log n$ for some $c \ge 1$.
  Then for any $\tau \ge 1$:
  $
  \sum_{t \in [\tau]}
  \E\|\q^t - \p^t \|_1
  \;\le\;
  O\big(
  \kappa^\tau \cdot
  \sqrt{\frac{m c \log n}{n}}
  +
  \frac{3\tau}{n^c}
  \big)
  $.
\end{restatable}

We start by recalling the overview of the argument,
which was briefly introduced in
Section~\ref{sec:tech-overview:framework}
First, recall by the law of iterated expectation that
for each $t \in [\tau]$: 
\begin{equation*}
  \E \|\q^{t+1} - \p^{t+1}\|_1
   \;=\;
  \E\big[\; \E_{t}\|\q^{t+1} - \p^{t+1}\|_1 \;\big]  \;.
\end{equation*}
Then by the triangle inequality and linearity of expectation,
it follows that
\begin{align}
  \sum_{t \in [T]}
  \E \|\q^{t+1} - \p^{t+1}\|_1
  &\;\le\;
    \E\Big[\;
    \sum_{t \in [T]}
    \E_{t}\|\q^{t+1} - \hatp^{t+1}\|_1
    +
    \E_{t}\|\hatp^{t+1} - \p^{t+1}\|_1
    \Big]
    \nonumber \\
  &\;=\;
    \E\Big[\;
    \sum_{t \in [T]}
    \E_{t}\|\q^{t+1} - \hatp^{t+1}\|_1
    +
    \|\hatp^{t+1} - \p^{t+1}\|_1
    \Big]
    \;.
  \label{error-tri}
\end{align}
Here, the final equality is due to the fact that
both $\hatp^{t+1}$ and $\p^{t+1}$ are
functions of $\{\pt\}$ and $\{\gt\}$, which means
$\E_t \|\hatp^{t+1} - \p^{t+1}\|_1 =
\|\hatp^{t+1} - \p^{t+1}\|_1$ for each $t$.
Thus in expression \eqref{error-tri},
we have decomposed the error (in conditional expectation) at
each round $t+1$ as the sum of the distances between
$\q^{t+1}$ and $\hatp^{t+1}$ and $\hatp^{t+1}$ and $\p^{t+1}$.
As mentioned in Section~\ref{sec:tech-overview:framework},
we control each of these sets of terms as follows:
\begin{itemize}[
  topsep=0.5em
  ]
\item
  To bound each $\|\q^{t+1} - \p^{t+1}\|_1$,
  recall thse distributions are related (coordinate-wise)
  are related under the randomness of $\gt$ and the
  same zero-sum family $\calF = \{F_j\}_{j \in [m]}$. 
  Thus if $\qt$ and $\pt$ are close, we can show
  $\q^{t+1}$ and $\hatp^{t+1}$ must also be close.
\item
  For terms  $\|\hatp^{t+1} - \p^{t+1}\|$,
  recall that $\p^{t+1}$ can be viewed as the
  empirical average of $n$ i.i.d samples from
  $\hatp^{t+1}$. Thus using the fact that
  $\hatp^{t+1}$ is a discrete probability distribution,
  we can use standard concentration bounds
  to control the error $\|\hatp^{t+1} - \p^{t+1}\|$.
\end{itemize}  
We make this intuition precise via the following
two propositions:

\begin{restatable}{proposition}{propphatqerr}
  \label{prop:phat-q-err}
  For every $t \ge 1$:
  $\;
  \E_t \big\|
  \hatp^{t+1} - \q^{t+1}
  \big\|_1
  \;\le\;
  (2+L)\cdot \E_{t-1} \big\| \pt - \qt \big\|_1
  $.
\end{restatable}

\begin{restatable}{proposition}{propphatpclose}
  \label{prop:p-hatp-close}
  Fix $\tau \ge 1$ and $c \ge 1$.
  For $n \ge 3c \log n$, it holds for all 
  $t \in [\tau]$ simultaneously that 
  \begin{align*}
    &\text{(i)}\quad
      \big| p^t_j - \hat p^t_j\big|
      \;\le\;
      \sqrt{\frac{3c\log n}{n}}
      \quad
      \text{for any fixed $j \in [m]$} \;,\\
    \text{and}\quad
    &\text{(ii)}\quad
      \big\|
      \pt - \hatpt
      \big\|_1
      \;\le\;
      \sqrt{\frac{9m \cdot c \log (2n)}{n}} \;,
  \end{align*}
  with probability
  at least $ 1- \frac{3\tau}{n^c}$.
\end{restatable}

Granting both propositions true for now,
we can then prove the main lemma:

\begin{proof}[Proof (of Lemma~\ref{lem:coupling-error})]
  Fix $c \ge 1$ and assume $n \ge 3 c\log n$.
  For readability, let us define
  \begin{equation*}
    \Phi := \frac{\sqrt{9m c\log (2n)}}{\sqrt{n}}
    \;.
  \end{equation*}
  Then by substituting the bound of claim (ii) of
  Proposition~\ref{prop:p-hatp-close}
  into expression~\eqref{error-tri}, we find that
  \begin{align}
    \E_t \| \q^{t+1} - \p^{t+1}\|_1
    &\;\le\;
      \E_{t}\big\|\q^{t+1} - \hatp^{t+1}\big\|_1
      +
      \big\|\hatp^{t+1} - \p^{t+1}\big\|_1 \nonumber \\
    &\;\le\; 
      \E_{t}\big\|\q^{t+1} - \hatp^{t+1}\big\|_1
      +
      \Phi
      \label{eq:error-1}
  \end{align}
  for all $t \in [\tau]$ simultaneously
  with probability at least $1 - \frac{3\tau}{n^c}$.
  Then substituting the bound of
  Proposition~\ref{prop:phat-q-err}
  into expression~\eqref{eq:error-1},
  for each $t$ we find  that
  \begin{equation*}
    \sum_{t \in [\tau]}
    \E_t \| \q^{t+1} - \p^{t+1}\|_1
    \;\le\;
    (2+L)\cdot \E_{t-1} \|\q^{t} - \p^{t}\|_1
    +
    \Phi
  \end{equation*}
  simultaneously with probability at least
  $1 - \frac{3\tau}{n^c}$.
  Now recall by definition that $\p^0 = \q^0$,
  which implies $\E_0[\q^1] = \E_0[\hatp^1]$.
  Then unrolling the recurrence yields
  \begin{align*}
    \E_t \| \q^{t+1} - \p^{t+1}\|_1
    \;\le\;
    (3+L)^t \cdot
    \Phi
  \end{align*}
  for each $t \in [\tau]$,
  again with probability at least $1 - \frac{3\tau}{n^c}$.
  Reindexing and summing over all $t$ yields
  with probability at least $1 - \frac{3\tau}{n^c}$: 
  \begin{equation*}
    \sum_{t \in [\tau]}
    \E_{t-1}\|\qt - \pt\|_1
    \;\le\;
    \sum_{t \in [\tau]}
    (3+L)^{t-1} \cdot 
    \Phi
    \;\le\;
    (3+L)^\tau \cdot 
    \Phi
  \end{equation*} 
  Finally, taking expectations and recalling
  the definitions of $\Phi$ and $\kappa := 3+L$,
  we conclude
  \begin{equation*}
    \sum_{t \in [\tau]}
    \E\|\qt - \pt\|_1
    \;\le\;
    \kappa^\tau \cdot 
    \sqrt{\frac{9 c m\log (2n)}{n}}
    + 
    \frac{3\tau}{n^c} \;.
    \qedhere
  \end{equation*}
\end{proof}

It now remains to prove Propositions~\ref{prop:phat-q-err}
and~\ref{prop:p-hatp-close}, which we do in the following subsections.

\subsection{Proof of Proposition~\ref{prop:phat-q-err}}

For convenience, we restate the proposition:

\propphatqerr*

\begin{proof}
  Recall by definition that
  \begin{align*}
    q_j^{t+1}
    &\;=\;
      q_j^{t}\cdot
      \big(
      1 + 
      F_j(\qt, \gt)
      \big) \\
    \text{and}\;\;
    \widehat p^{t+1}_j
    &\;=\;
      p_j^{t}\cdot
      \big(
      1 +
      F_j(\pt, \gt)
      \big)
  \end{align*}
  for all $j \in [m]$.
  For readability we will write $\hatp', \p', \q'$ for  $\hatp^{t+1}$,
  $\p^{t+1}$, $\q^{t+1}$,  and $\p$, $\q$, $\g$ for $\pt$, $\qt$, $\gt$,
  respectively. It follows that
  \begin{align*}
    \E_{t}
    \big\|
      \hatp' - \q' \big\|_1
    &\;=\;
      \sum_{j \in [m]}
      \E_{t} \big|
      p_j - q_j
      +
      p_j \cdot F_j(\p, \g)
      -
      q_j \cdot F_j(\q, \g)
      \big|
      \nonumber \\
    &\;\le\;
      \sum_{j \in [m]}
      \E_t | p_j - q_j |
      +
      \E_t
      |(p_j - q_j)\cdot F_j(\p, \g)|
      + 
      \E_t
      |q_j \cdot (F_j(\p, \g) - F_j(\q, \g))|
      \nonumber \\
    &\;\le\;
      \E_t \| \p - \q\|_1
      +
      \sum_{j \in [m]}
      \E_t |p_j - q_j |
      +
      q_j \big( L \cdot \E_t\|\p - \q\|_1 \big)
      \nonumber \\
    &\;=\;
      (2 + L) \cdot \E_t \|\p - \q\|_1 \;.
  \end{align*}
  Here, the first line follows from two applications of
  the triangle inequality, and the second line comes from
  applying the boundedness and $L$-Lipschitz property
  of each $F_j$ from
  Definition~\ref{def:zsmwu} and part (ii) of 
  Assumption~\ref{ass:F-params}.

  Finally, because $\pt = \p$ and $\qt = \q$ are functions
  only of $\{\p^{t-1}\}$ and $\{\g^{t-1}\}$, it follows that
  $\E_t \| \p - \q \|_1 = \E_{t-1} \|\p - \q\|$. Thus we conclude that
  $\E_t\|\hatp' - \q'\|_1 \le (2 + L) \cdot \E_{t-1} \|\p-\q\|_1$.
\end{proof}

\subsection{Proof of Proposition~\ref{prop:p-hatp-close}}

For convenience, we restate the proposition:

\propphatpclose*

\begin{proof}
  We start by proving claim (i). 
  Using a standard multiplicative Chernoff bound
  \cite[Corollary 4.6]{DBLP:books/daglib/0012859},
  we have for each $j \in [m]$ and $t \in [\tau]$ that
  \begin{equation*}
    \Pr_{t-1}\Big(\;
    \Big|\; p^t_j - \E_{t-1}[p^t_j]\;\Big| \ge \E_{t-1}[p^t_j] 
    \cdot \delta
    \Big)
    \;\le\;
    2\cdot
    \exp\Big(
    -\frac{n}{3}\cdot \E_{t-1}[p^t_j] \cdot \delta^2
    \Big) \;,
  \end{equation*}
  for any $0 < \delta \le 1$.
  Fix $c \ge 1$, and consider the case when
  $\sqrt{\tfrac{3c \log n}{n}} \le \E_{t-1}[p^t_j] \le 1 $.
  Then setting
  $\delta = \tfrac{1}{\E_{t-1}[p^t_j]} \cdot
  \sqrt{\tfrac{3c \log n}{n}} \le 1$ implies 
  \begin{equation*}
    \Pr_{t-1}\Big(\;
    \Big|\; p^t_j - \E_{t-1}[p^t_j]\;\Big|
    \ge 
    \sqrt{\tfrac{3c\log n }{n}}
    \Big)
    \;\le\;
    2 \cdot \exp\Big(
    \tfrac{- c \log n}{\E_{t-1}[p^t_j]}
    \Big)
    \;\le\;
    \frac{2}{n^c} \;.
  \end{equation*}
  On the other hand, when
  $0 \le \E_{t-1}[p^t_j]< \sqrt{\tfrac{3c\log n}{n}}$,
  setting $\delta = 1$ implies
  \begin{equation*}
    \Pr_{t-1}\Big(\;
    \Big|\; p^t_j - \E_{t-1}[p^t_j]\;\Big|
    \ge 
    \sqrt{\tfrac{3c\log n }{n}}
    \Big)
    \;\le\;
    2\cdot
    \exp\Big(
    - \tfrac{1}{3} \cdot \sqrt{3c n \log n}
    \Big)
    \;\le\;
    \frac{2}{n^c}\;,
  \end{equation*}
  where the final inequality holds for all $n \ge 3c \log n$.
  Summing over all $\tau$ rounds and taking a union bound
  proves that statement (i) of the proposition
  holds with probability at least $1 - 2\tau/n^c$. 

  To prove statement (ii), one approach would be to
  simply take a union bound over all $m$ coordinates,
  which would result in a linear dependence on $m$.
  However, recall that each $\pt$ is a discrete
  $m$-dimensional distribution, and $\hatpt$ is
  by definition the empirical average of
  $n$ i.i.d. samples from $\pt$. Thus to bound
  each $\|\pt - \hatpt\|$, we use a standard tool
  from discrete distribution learning
  (e.g., of Cannone~\cite[Theorem 1]{canonne2020short}),
  which states:
  \begin{equation*}
    \Pr_{t-1}\Big[
    \|\pt - \hatpt\|_1
    > 2\eps + \sqrt{\tfrac{m}{n}}
    \Big]
    \;\le\;
    \frac{1}{n^c}
  \end{equation*}
  so long as $n \ge \max\big\{\frac{m}{\eps^2},
  \frac{2c}{\eps^2}\log (2n)\big\}$.
  Then setting $\eps := \sqrt{\frac{m c\log n}{n}}$
  ensures that for $n \ge 3c \log n$,
  each $\|\pt - \hatpt\|_1 \le 3\sqrt{\frac{m c \log(2n)}{n}}$
  with probability at least $1 - 1/n^c$.
  Then taking a union bound over all $\tau$ rounds means
  the claim holds simultaneously for all $t \in [\tau]$
  with probability at least $1 - \tau/n^c$.

  Finally, taking a union bound over claims (i) and (ii)
  of the proposition ensures both claims hold
  simultaneously with probability at least $1 - 3\tau/n^c$.
\end{proof}

% proof of stationary regret

\section{Proof of Regret and Consensus for Stationary Rewards}
\label{sec:stationary-regret-proof}

In this section we develop the proof of
our main results, Theorem~\ref{thm:stationary-regret}
and Theorem~\ref{thm:consensus},
which we restate here:

\stationaryregret*

\consensus*

As introduced in Section~\ref{sec:tech-overview:regret},
the proof of these results relies on decomposing
the $T$ rounds into three distinct phases and
bounding each of their individual regrets.
Before reviewing the exact setup of the phases,
in order to streamline the development of
the proof, we summarize the setting and assumptions of
Theorem~\ref{thm:stationary-regret}
and Theorem~\ref{thm:consensus} in the following
assumption:

\begin{assumption}
  \label{ass:stationary-proof}
  Consider the sequence $\{\pt\}$ from
  Definition~\ref{def:coupling} generated
  by the $\betaadopt$ protocol with
  a stationary reward sequence $\{\gt\}$.
  Assume $\p^1 = \1/m$ deterministically.
  As in Assumption~\ref{reward:stationary},
  assume that each $\gt \in [0, \sigma]^m$ has mean
  $\bfmu := (\mu_1, \dots, \mu_m)$,
  where all $\mu_j$ and $\sigma \ge 1$
  are asbolute constants.
  Without loss of generality assume
  $0 \le \mu_m \le \dots \le \mu_2 < \mu_1 \le 1$.
  Moreover, assume $\beta \in (0, \min(1/4, 1/\sigma)]$
  is an absolute constant,
  meaning by Proposition~\ref{prop:beta-adopt-params}
  that the family $\{F_j\}$ induced by
  the $\betaadopt$ protocol satisfies
  Assumption~\ref{ass:F-params} with parameters
  $\alpha = 2\beta$, $\delta =0$, and $L =2$.
  Finally, assume that $m = o(\log^{1.5} n)$. 
\end{assumption}

\paragraph{Recap of phase structure and overview
of proof:}
We now review the overview of the proof
of Theorem~\ref{thm:stationary-regret}
(and as a byproduct, the proof of Theorem~\ref{thm:consensus})
and the structure of the three phases
introduced in Section~\ref{sec:tech-overview:regret}.
Specifically, recall that we
define 
$\Tone :=\sqrt{n} \cdot \frac{\log m}{\sqrt{m \log n}}$
and
$\Ttwo := 2 \sqrt{n} \cdot \log n$,
and we define 
\begin{align*}
  \Rone
    &\;:=\;
      {\textstyle
      \; \sum_{t=1}^{\min(\Tone, T)}
      \mu_1 -\E[\langle \pt, \gt\rangle]
      } \;, \\
    \Rtwo
    &\;:=\;
      {\textstyle
      \; \sum_{t=\Tone}^{\min(\Ttwo, T)}
      \mu_1 - \E[\langle \pt, \gt\rangle]
      }
      \;\;\text{if $T \ge \Tone$, and 0 otherwise,} \\
    \Rthree
    &\;:=\;
      {\textstyle
      \; \sum_{t=\Ttwo}^{T}
      \mu_1 - \E[\langle \pt, \gt\rangle]
      }
      \;\;\;
      \text{if $T \ge \Ttwo$, and 0 otherwise},
\end{align*}
which means that $R(T) = \Rone + \Rtwo + \Rthree$.
We also introduce the two times
$\tbounded$ and $\tconsensus$ defined as follows:
\begin{align}
  \tbounded
  &\;:=\;
    \min\big\{
    n^2,\;
    \min\big\{
    t \in [T] \;:\;
    p^t_1 > 1- \frac{\log^2 n}{\sqrt{n}}
    \big\}
    \big\}  
    \label{eq:tbounded}
  \\
  \tconsensus
  &\;:=\;
    \min\big\{
    n^2,\;
    \min\big\{
    t \in [T] \;:\;
    p^t_1 = 1
    \big\}
    \big\} \;.
  \label{eq:tconsensus}
\end{align}
In other words:
\begin{itemize}[
  topsep=0.5em,
  itemsep=0em,
  ]
\item
  $\tbounded$ is the first round $t$
  such that the mass $p^t_1$ is bounded below
  by $1- \softO(1/\sqrt{n})$ for every
  subsequent round $\tbounded\le t \le n^2$.
\item
  $\tconsensus$ is the first round
  such that $p^t_1 = 1$ (meaning the
  population reaches and remains in consensus
  on the highest-mean action) 
  for every subsequent round $\tconsensus \le t \le T$. 
\end{itemize}

Note that the ``double minimum'' in the definition
of $\tbounded$ and $\tconsensus$ is to ensure
these two random variables are well-defined in the
(provably low-probability) event that
the inner-most set is empty.
Given these definitions,  the
proofs of Theorem~\ref{thm:stationary-regret}
and Theorem~\ref{thm:consensus}
then rely on the following sequence of
intermediate lemmas: 
\begin{itemize}[
  topsep=0.5em,
  itemsep=0em,
  ]
\item
  \textbf{Phase 1 -- Lemma~\ref{lem:reg-phase-1}}:
  We prove $\Rone \le O(\sigma \log m)$
  by applying the $\tau$-step regret bound of
  Proposition~\ref{prop:tau-step-regret} over
  a sequence of $D$ epochs, where $D \approx \Tone$.
\item
  \textbf{Phase 2 -- Lemma~\ref{lem:reg-phase-2}}:
  We prove $\Rtwo \le O(\log^3 n)$ by showing first that,
  with high probability, $\tbounded \le \Tone$.
  This implies that the per-round regret for each
  round $\Tone \le t \le \Ttwo$ is at most $O((\log^2 n)/\sqrt{n})$,
  of which there are at most $O(\sqrt{n} \cdot \log n)$ rounds.
\item
  \textbf{Phase 3 -- Lemma~\ref{lem:reg-phase-3}}:
  We prove that $\Rthree = 0$ with probability at least
    $1 - O(1/\sqrt{n})$. 
    This is established by first showing
    in Lemma~\ref{lem:t-consensus}
    that, with this same probability, $\tconsensus \le \Ttwo$.
    In other words, the population
    reaches and remains in consensus on the highest-mean
    action within $\Ttwo \le O(\sqrt{n} \log n)$  rounds with
    high probability, and as a consequence,
    the cumulative regret $\Rthree$
    over all subsequent rounds $t \ge \Ttwo$ is 0. 
\item
  \textbf{High probability consensus -- via Lemma~\ref{lem:t-consensus}}:
  The result of Theorem~\ref{thm:consensus} is then
  a direct consequence of Lemma~\ref{lem:t-consensus},
  which shows that $\tconsensus \le O(n \log^2 n)$ with high probability. 
\item
  \textbf{Core tool: sufficient growth of $p^t_1$ -- Lemma~\ref{lem:stat-util}}:
  The core tool used in each of the previous intermediate results
  is Lemma~\ref{lem:stat-util}, where we establish probabilistic
  guarantees on the growth of the best action's mass $p^t_1$
  over time. In particular, this allows us to establish
  the tightest possible regret guarantees for $\Rone$
  in Lemma~\ref{lem:reg-phase-1}, and to establish
  the probabilistic claims that $\tbounded \le \Tone$
  and $\tconsensus \le \Ttwo$ used in
  phases 2 and 3 respectively. 
\end{itemize}

Given this overview, we begin by proving the
sufficient growth guarantees of Lemma~\ref{lem:stat-util},
and we develop the proofs of the remaining intermediate
lemmas in the subsequent subsections. 

\subsection{Sufficient Growth of $p^t_1$}

\begin{lemma}
  \label{lem:stat-util}
  Consider the setting of Assumption~\ref{ass:stationary-proof}.
  Then the following statements hold simultaneously with
  probability at least $1-3/n$ for sufficiently large $n$:
  \begin{align*}
    &(i)\;\;\;
      \text{for all $t \in [n]$}:\;
      p^{t+1}_1 \;\ge\;
      \min \Big\{p^t_1\;,\; 1-\frac{\log^2 n}{\sqrt{n}} \Big\}
    \\
    &(ii)\;\;\;
      \text{for all $t \in [\tbounded]$}:\;
      p^{t+1}_1 \;\ge\;
      \frac{1}{2m}
      \Big(
      1 + \frac{\beta (\mu_1-\mu_2)\cdot\log^2 n}{\sqrt{n}}
      \Big)^t \\
    &(iii)\;\;\;
      \tbounded \;\le\;
      \frac{3}{\beta(\mu_1-\mu_2)}\cdot
      \frac{\sqrt{n} \log m}{\log^2 n} \;.
  \end{align*}
\end{lemma}

\begin{proof}
  Based on the settings of
  Assumption~\ref{ass:stationary-proof},
  we have that the $\betaadopt$ protocol 
  satisfies Assumption~\ref{ass:F-params}
  with parameters $\alpha := 2\beta$ and $\delta = 0$.
  It follows that in the stationary reward setting,
  we have in conditional expectation that
  \begin{equation*}
    \E_{\p^t}[p^{t+1}_1]
    \;\ge\;
    p^t_1 \big(
    1 + \tfrac{\alpha}{2}(\mu_1 - \langle \pt, \bfmu \rangle)
    \big)
    \;\ge\;
    p^t_1 \big(
    1 + \beta(\mu_1 - \mu_2)
    \cdot (1-p^t_1)
    \big) \;.
  \end{equation*}
  For readability, let $\gamma := \beta(\mu_1 -\mu_2)$.
  Now by statement (i) of Proposition~\ref{prop:p-hatp-close},
  we have with probability at least $1-3n^2/n^3 = 1-3/n$
  simultaneously for all $t \in [n]$ that
  $p^{t+1}_1 \ge \E_{\pt}[p^{t+1}] -\sqrt{\tfrac{9 \log n}{n}}$.
  We can therefore write:
  \begin{equation}
    p^{t+1}_1
    \;\ge\;
    p^t_1 \big(
    1 + \gamma
    \cdot (1-p^t_1)
    \big)
    - \sqrt{\frac{9 \log n}{n}}
    \;=\;
    p^t_1 
    + \gamma \cdot p^t_1(1-p^t_1)
    - \sqrt{\frac{9 \log n}{n}} \;.
    \label{eq:pt1-lb}
  \end{equation}

  We now start by proving part (i) of the lemma.
  For this, observe that the right hand side
  of~\eqref{eq:pt1-lb} is concave and increasing in $p^t_1$.
  Thus suppose $p^t_1 = 1 - ((\log^2 n)/\sqrt{n})$ 
  and that the right hand side of \eqref{eq:pt1-lb}
  is at least $1 - ((\log^2 n)/\sqrt{n})$
  (which means $p^{t+1}_1$ is also at least this large).
  Then this implies that 
  $p^{t+1}_1 \ge 1 - ((\log^2 n)/\sqrt{n})$ when 
  initially $p^{t}_1 > 1 - ((\log^2 n)/\sqrt{n})$.
  Thus to prove part (i) of the lemma, it suffices to check
  that $p^{t+1}_1 \ge p^t_1$ when
  $1/m \le p^t_1 \le 1 - ((\log^2 n)/\sqrt{n})$.
  For this, observe by the concavity of $p^t_1(1-p^t_1)$
  and the constraints on $p^t_1$ that
  \begin{equation*}
    \gamma \cdot p^t_1(1-p^t_1) - \sqrt{\frac{9\log n}{n}}
    \;\ge\;
      \gamma \cdot
      \Big(
      \frac{\log^2 n}{\sqrt{n}} - \frac{4 \log^2 n}{n}
      \Big)
      - \sqrt{\frac{9c\log n}{n}} 
    \;\ge\;
      \gamma \cdot \frac{\log^2 n}{2\sqrt{n}}
  \end{equation*}
  for sufficiently large $n$.
  Substituting this back into equation~\eqref{eq:pt1-lb}
  and using the fact that $\gamma > 0$ yields
  $p^{t+1}_1 \ge p^t_1$, which
  finishes the proof for part (i) of the lemma.

  To prove part (ii), recall by definition of $\tbounded$
  that $p^t_1 \le 1 - \frac{\log^2 n}{\sqrt{n}}$
  for all $t \in [\tbounded]$. Then reusing the calculation
  from expression~\eqref{eq:pt1-lb}, we have
  \begin{equation}
    p^{t+1}_1
    \;\ge\;
    p^t_1 \Big(
    1+ \gamma\cdot \frac{\log^2 n}{\sqrt{n}}
    \Big)
    -
    \sqrt{\frac{9\log n}{n}}
    \label{eq:pt1-2}
  \end{equation}
  for all $t \in [\tbounded]$ with probability at least $1-3/n$.
  Now for readability let
  $\lambda := \gamma (\log^2 n)/\sqrt{n}$,
  and let $\Phi := \sqrt{(9\log n) /n }$,
  and recall that $p^0_1 = 1/m$.
  Then unrolling the recurrence in~\eqref{eq:pt1-2}, we find
  \begin{align*}
    p^{t+1}_1
    &\;\ge\;
      \frac{1}{m}(1+\lambda)^t
      -
      \Phi \cdot \bigg(
      \sum_{i=0}^{t-1} (1+\lambda)^i
      \bigg) \\
    &\;=\;
      \frac{1}{m}(1+\lambda)^t
      -
      \Phi \cdot \bigg(
      \frac{(1+\lambda)^t - 1}{\lambda}
      \bigg) \\
    &\;\ge\;
      \frac{1}{m}(1+\lambda)^t
      -
      \Phi \cdot \bigg(
      \frac{(1+\lambda)^t}{\lambda}
      \bigg) 
      \;=\;
      (1+ \lambda)^t \cdot
      \bigg(\frac{1}{m} - \frac{\Phi}{\lambda} \bigg) \;,
  \end{align*}
  where in the second line we apply the definition
  of a finite geometric series. 
  Moreover, we have
  $\Phi/\lambda = 3/(\gamma \log^{1.5}n)\le 1/(2m)$,
  where the inequality holds when $m = o(\log^{1.5} n)$.
  Then together with the definitions of $\gamma$ and $\lambda$,
  we have for all $t \in [\tbounded]$ that
  \begin{equation}
    p^{t+1}
    \;\ge\;
    \frac{1}{2m}
    \Big(
    1 + \gamma
    \cdot \frac{\log^2 n}{\sqrt{n}}
    \Big)^t
    \;,
    \label{eq:pt1-3}
  \end{equation}
  which proves part (ii) of the lemma.
  
  To prove part (iii), we use the lower bound in
  equation~\eqref{eq:pt1-3} and the fact that $1+x \ge e^{x/2}$
  for all $x \in (0, 1)$ to write
  \begin{equation*}
    p^{\tbounded}
    \;\ge\;
    \frac{1}{2m} \cdot
    \exp\Big(
    (\tbounded-1) \cdot
    \frac{\gamma \log^2 n}{\sqrt{n}}
    \Big)
    \;.
  \end{equation*}
  Then it follows that
  \begin{equation*}
    \tbounded
    \;=\;
    \frac{\sqrt{n}}{\gamma \log^2 n}
    \cdot \log\Big(2m\Big(1 -\frac{\log^2 n}{\sqrt{n}}\Big)\Big)
    + 1
    \;\le\;
    \frac{3}{\beta (\mu_1-\mu_2)}
    \cdot
    \frac{\sqrt{n} \log m}{\log^2 n}
  \end{equation*}
  is sufficient to ensure that 
  $p^{\tbounded+1} \ge 1 - ((\log^2 n)/\sqrt{n})$
  with probability at least $1- 3/n$.
  This proves part (iii) of the lemma.
\end{proof}

\subsection{Phase 1 Regret}

We now derive the following bound
on the phase 1 regret $\Rone$:

\begin{lemma}
  \label{lem:reg-phase-1}
  Consider the setting of
  Assumption~\ref{ass:stationary-proof}. 
  Then $\Rone \le O(\sigma \log m)$. 
\end{lemma}

As mentioned in the overview,
the proof of this lemma uses the $\tau$-step regret
bound of Proposition~\ref{prop:tau-step-regret},
which we restate here for reference:

\taustep*

Observe that the second term in the bound above
(stemming from Lemma~\ref{lem:coupling-error})
can in the worst case grow exponentially with $\tau$.
To obtain sharper bounds, for large values of $\tau$,
we can run multiple ``epochs'' of the coupling (similar to
the analysis of Celis et al.~\cite{DBLP:conf/podc/CelisKV17})
where at the start of each epoch, the
process $\{\qt\}$ is reinitialized to the
most recent $\pt$. 

\begin{proof}
  Recall that we define
  $\Rone := \sum_{t=1}^{\min(\Tone, T)}
  \mu_1 -\E[\langle \pt, \gt\rangle]
  $
  where
  $\Tone := \sqrt{n} \cdot \frac{\log m}{\sqrt{m \log n}}$.
  For readability, let $\widehat T := \min(\Tone, T)$.
  By Proposition~\ref{prop:regret-decompose},
  observe that we can write 
  \begin{equation*}
    \Rone
    \;\le\;
    \sum_{t \in [\widehat T]}
    \mu_1 -\E[\langle \qt, \gt\rangle]
    +
    \sigma \sum_{t\in [\widehat T]}
    \E\|\qt - \pt \|_1
    \;:=\;
    \Rhatone
  \end{equation*}
  To bound the quantity $\Rhatone$, we will repeatedly
  apply the general $\tau$-step regret bound of
  Proposition~\ref{prop:tau-step-regret} over
  a sequence of $D$ epochs. 
  In particular, we assume
  each epoch $d \in [D]$ consists of $\tau$ rounds
  with the following properties:
  \begin{enumerate}[
    % topsep=0pt,
    itemsep=0pt,
    leftmargin=3em,
    label={(\roman*)}
    ]
  \item
    Let $D = \widehat T / \tau$. 
    For each $d \in [D]$, let epoch $\calE_d$
    be the set of rounds
    $\calE_d := \{d\tau, \dots, (d + 1)\tau - 1\}$.
  \item
    Set $\p^0 = \q^0 = \mathbf{1} \cdot \tfrac{1}{m}$.
  \item
    For $d \ge 1$, set $\q^{d\tau} = \p^{d\tau}$.
  \item
    Then in each epoch, run the coupling as in
    Definition~\ref{def:coupling} for $\tau$ rounds,
    until $\widehat T$ total rounds have passed. 
  \end{enumerate}
  In other words, at the start of each epoch,
  the trajectory of $\qt$ is initialized at the
  most recent point $\pt$ from the end
  of the previous epoch. This allows for a tighter coupling of
  $\{\qt\}$ and $\{\pt\}$ over all $T$ rounds,
  as we can guarantee that the trajectories stay closer
  when the number of rounds $\tau$ is smaller than $T$.
  Note that these epochs are defined purely for obtaining
  sharper regret bounds for the sequence $\{\pt\}$,
  and that this sequence $\{\pt\}$ is never restarted. 

  Using the general $\tau$-step regret bound from
  Proposition~\ref{prop:tau-step-regret}, we can then
  bound the overall regret $\Rhatone$ by the sum of
  (the upper bounds on) the regret of each epoch.
  In particular, we define for each $d \in [D]$
  \begin{equation*}
    \widehat R(\calE_d)
    \;=\;
    \sum_{t \in \calE_d} \mu_1
    - 
    \E[ \langle \qt, \gt \rangle]
    +
    \sum_{t \in \calE_d}
    \sigma \E\| \qt - \pt \|_1 \;,   
  \end{equation*}
  meaning that
  $\Rone \le \Rhatone = \sum_{d \in [D]} \widehat R(\calE_d)$.
  Now recall from Lemma~\ref{lem:stat-util} that
  \begin{equation}
    p^{t+1}_1 \;\ge\;
    \frac{1}{2m}
    \Big(
    1 + \frac{\beta (\mu_1-\mu_2)\cdot\log^2 n}{\sqrt{n}}
    \Big)^t
    \label{eq:r1-p1-bound}
  \end{equation}
  simultaneously for all $t \in [\Tone]$
  with probability at least $1-3/n$.
  
  Using the regret bound of Theorem~\ref{thm:zsmwu-regret}
  under the stationary reward setting of
  Assumption~\ref{ass:stationary-proof}, 
  it follows that
  \begin{equation*}
    \sum_{t \in \calE_d}
    \mu_1 - \E[\langle \qt, \gt \rangle]
    \;\le\;
    O\bigg(
    \log\bigg(
    \frac{2m}{
      (1 + ((\log^2 n)/\sqrt{n})^{\tau d})
    }
    \bigg)
    \bigg)
    +
    O\bigg(
    \frac{\tau}{n}
    \bigg) \;.
  \end{equation*}
  Here, the second term in the bound comes from the fact that
  with remaining probability at most $3/n$, the
  guarantee in expression~\eqref{eq:r1-p1-bound}
  fails to hold. 
  Thus in the worst case, 
  $\mu_1 - \langle \qt, \gt\rangle \le 1$
  for each $t \in \calE_d$, which leads in expectation to
  an additional regret of at most $O(\tau/n)$ during
  the epoch.
  Then applying the $\tau$-step regret bound of 
  Proposition~\ref{prop:tau-step-regret} to each
  $\widehat R(\calE_d)$ (and given the value of $\Tone$)
  and simplifying, it follows that
  \begin{align}
    \sum_{d \in [D]} \widehat R(\calE_d)
    &\;\le\;
      O\bigg(
      \sum_{d \in [D]} \log (2m)
      -
      \log \Big(1 + \frac{\log^2 n}{\sqrt{n}}\Big)^{\tau d}
      +
      \frac{\tau}{n}
      \bigg)
      + 
      O\bigg(
      D \sigma
      \bigg(
      \kappa^\tau
      \sqrt{\frac{2m \log n}{n}}
      +
      \frac{1}{n}
      \bigg)
      \bigg) \nonumber \\
    &\;\le\;
      O\bigg(
      D \log(2m)
      -
      \tau D^2
      \log \Big(1 + \frac{\log^2 n}{\sqrt{n}}\Big)
      +
      \frac{D\tau}{n}
      \bigg)
      +
      O\bigg(
      D \sigma
      \kappa^\tau
      \sqrt{\frac{2m \log n}{n}}
      \bigg) \;.
      \label{eq:epoch-1}
  \end{align}
  We now set $\tau = \Theta(1)$, which means
  $D \le \widehat T \le \Tone
  =\sqrt{n} \cdot \frac{\log m}{\sqrt{m \log n}}$.
  It follows that, for sufficiently large $n$,
  the first term of expression~\eqref{eq:epoch-1}
  is at most $O(\log m)$, and the second term
  is at most $O(\sigma \log m)$.
  Thus $\Rone \le \Rhatone =
  \sum_{d \in [D]} \widehat R(\calE_d) \le O(\sigma \log m)$, 
  which concludes the proof.
\end{proof}

\subsection{Phase 2 Regret}

Using the probabilistic
guarantees of Lemma~\ref{lem:stat-util},
we show in the following lemma that 
$\tbounded \le \Tone$,
which in turn implies a bound on $\Rtwo$:

\begin{lemma}
  \label{lem:reg-phase-2}
  Consider the setting of
  Assumption~\ref{ass:stationary-proof}.
  Then $\tbounded \le \Tone$ with
  probability at least $1-O(1/n)$. Moreover, 
  $\Rtwo \le O(\log^3 n)$.
\end{lemma}

\begin{proof}
  Assuming that $T \ge \Tone$, recall that 
  $
  \Rtwo := 
  \sum_{t=\Tone}^{\min(\Ttwo, T)}
  \mu_1 -\E[\langle \pt, \gt\rangle]
  $,
  where $\Ttwo := 2\sqrt{n}\cdot \log n$.
  To derive an upper bound on $\Rtwo$,
  observe that for each $t$,
  we have
  \begin{equation}
    \mu_1 - \E[\langle \pt, \gt\rangle]
    \le
    (\mu_1 - \mu_2)\cdot\E[(1-p^t_1)] \;,
    \label{eq:regp2-1}
  \end{equation}
  which follows by the fact that
  the randomness in $\gt$ is independent
  from that of $\pt$, and from
  the stationary reward setting
  structure of Assumption~\ref{ass:stationary-proof}.

  Now recall from Lemma~\ref{lem:stat-util} that
  simultaneously with probability at least $1-3/n$:
  \begin{align*}
    \tbounded
    \;\le\;
      \frac{3}{\beta(\mu_1-\mu_2)}
      \cdot \frac{\sqrt{n} \log m}{\log^2 n}
    \quad\text{and}\quad
    p^t_1
    \ge 1 - \frac{ \log^2 n}{\sqrt{n}}
    \quad
    \text{for all $\tbounded \le t \le n$} \;.
  \end{align*}
  Observe further that by definition of $\Tone$,
  \begin{equation*}
    \tbounded
    \;\le\;
    \frac{3}{\beta(\mu_1-\mu_2)}
    \cdot \frac{\sqrt{n} \log m}{\log^2 n}
    \;\le\;
    \sqrt{n}\cdot \frac{\log m}{\sqrt{m \log n}}
    \;=\;
    \Tone
  \end{equation*}
  so long as $m \le \frac{\beta(\mu_1 - \mu_2) \log^3 n}{3}$,
  which holds under the stronger assumption
  that $m = o(\log^{1.5} n)$.
  Thus $\tbounded \le \Tone$ with probability at
  least $1-3/n$, which proves the first statement
  of the lemma. 

  To prove the second statement,
  observe that the first statement and the
  result of Lemma~\ref{lem:stat-util} together imply that,
  with probability at least $1-O(1/n)$, 
  each  $p^t_1 \ge 1 - (\log^2 n)/\sqrt{n}$
  simultaneously for all $\Tone \le t \le \Ttwo$.
  In particular, this implies that
  \begin{equation*}
    \E[(1-p^t_1)]
    \;\le\;
    O\bigg(\frac{\log^2 n}{\sqrt{n}} + \frac{1}{n}\bigg)
  \end{equation*}
  for all $t$ in this range. 

  Moreover, observe by definition of $\Tone$ and $\Ttwo$
  that
  \begin{equation*}
    \min(\Ttwo, T) - \Tone
    \;\le\;
    \Ttwo - \Tone
    \;=\;
    2\sqrt{n} \cdot \log n
    -
    \sqrt{n} \cdot \frac{\log m }{\sqrt{m \log n}}
    \;\le\;
    \sqrt{n}  \cdot \log n\;.
  \end{equation*}
  Based on the bound in expression~\eqref{eq:regp2-1},
  it then follows with probability at least $1-O(1/n)$ that
  \begin{equation*}
    \sum_{t=\Tone}^{\min(\Ttwo, T)}
    (\mu_1 - \mu_2)
    \E[(1 - p^t_1)]
    \;\le\;
    O\bigg(
    \sqrt{n} \log n \cdot
    \bigg(\frac{\log^2 n}{\sqrt{n}}
    + \frac{1}{n}
    \bigg)
    \bigg)
    \;\le\;
    O\big(
    \log^3 n
    \big) \;.
  \end{equation*}
\end{proof}

\subsection{Phase 3 Regret}

We now prove the following bound on $\Rthree$:

\begin{lemma}
  \label{lem:reg-phase-3}
  Consider the setting of Assumption~\ref{ass:stationary-proof}.
  Then $\Rthree = 0$ with probability at least $1-O(1/\sqrt{n})$.
\end{lemma}

The proof of this lemma relies on two steps,
similar to those of Lemma~\ref{lem:reg-phase-2}.
The key first step is to show that $\tconsensus$
(the first time such that $p^t_1 = 1$)
occurs before the start of phase 3 with high probability.
Conditioned on this event, the regret in each round
$\Ttwo \le t \le T$ is 0, and thus $\Rthree = 0$.
In particular, we first prove the
following lemma:

\begin{lemma}
  \label{lem:t-consensus}
  Consider the setting of Assumption~\ref{ass:stationary-proof}.
  Then $\tconsensus \le \Ttwo$ with probability
  at least $1-O(1/\sqrt{n})$. 
\end{lemma}

The proof of Lemma~\ref{lem:t-consensus}
leverages the following result of
Lengler~\cite{lengler2020drift},
which bounds the expected hitting time of
a real-valued random process with variable drift:
\begin{theorem}[{\cite[Theorem 3]{lengler2020drift}}]
  \label{thm:drift}
  Let $\{y^t\}$ be a sequence of non-negative
  random variables with state space $\{0, \dots, n\}$.
  Let $\tau := \min\{t \ge 0 : y^t = n\}$.
  Assume there exists a decreasing function
  $\Delta: [n] \to \R$ such that
  \begin{equation*}
    \E_{y^t=i}\big[ y^{t+1} - y^t\big] \ge \Delta(i) \;
  \end{equation*}
  for all $i \in \{y^0, \dots, n-1\}$. Then
  \begin{equation*}
    \E[\tau]
    \;\le\;
    \frac{1}{\Delta(n-1)}
    +
    \int_{y^0}^{n-1} \frac{1}{\Delta(s)} ds \;.
  \end{equation*}  
\end{theorem}

Equipped with this result, we now prove
Lemma~\ref{lem:t-consensus}:

\begin{proof}[Proof (of Lemma~\ref{lem:t-consensus})]
  We will use the  bound on expected 
  hitting times of drift processes from
  Theorem~\ref{thm:drift} to bound $\tconsensus$.
  For this we introduce the following notation:
  \begin{itemize}[
    % leftmargin=1em,
    itemsep=0em,
    ]
  \item
    Let $c^t_i \in [m]$ denote the action chosen by
    agent $i \in [n]$ in round $t \in [T]$.
  \item
    Let $x^t_j$ denote the number of agents choosing
    action $j \in [m]$ in round $t \in [T]$.
  \end{itemize}
  Clearly, we have $x^t_j := \sum_{i \in [n]} \1\{c^t_i = j\}$
  for all $j \in [m]$ and $\sum_{j \in [m]} x^t_j = n$
  for every $t \in [T]$.
  Moreover, by definition of $\pt$, we have
  $p^t_j = (x^t_j/n)$ for all $t \in [T]$
  and $j \in [m]$. With this notation, it follows that
  $\tconsensus$ (from~\eqref{eq:tconsensus})
  can be equivalently written as:
  \begin{equation*}
    \tconsensus
    :=
    \min
    \big\{
    n^2, \;
    \min\big\{
    t \in [T] \colon
    x^t_1 = n
    \big\}
    \big\} \;.
  \end{equation*}
  We will further write $\tconsensus = \Tone + \hattau$,
  where $\hattau$ is the number of additional rounds after
  round $\Tone$ until $x^t_1 = n$.
  Then recalling that
  $\Tone = \sqrt{n}\cdot \frac{\log m}{\sqrt{m \log n}}
  \le \sqrt{n} \cdot \log^3 n$,
  it suffices to show that
  $\hattau \le \sqrt{n} \cdot \log^3 n$ with probability
  at least $1 - O(1/\log n)$ 
  in order to prove that
  $\tconsensus \le \Ttwo = 2\sqrt{n} \log^3 n$
  with probability at least $1 - O(1/\log n)$.

  For this, recall that by combining the results
  of Lemma~\ref{lem:stat-util} and Lemma~\ref{lem:reg-phase-2},
  we have simultaneously with probability at least $1-O(1/n)$
  that $\tbounded \le \Tone$ and
  thus $x^t_1 \ge n - \sqrt{n} \cdot \log^2 n$
  for all $\Tone \le t \le n$.
  Conditioned on this event, which we denote by $\calA$,
  we now compute the expected drift
  $\E_{\calA, \pt}[x^{t+1}_i - x^t_i]$.
  In particular, we start by additionally conditioning on
  the reward vector $\gt$ and write
  \begin{align}
    \E_{\calA, t}[x^{t+1}_i]
    &\;=\;
      \E_{\calA, t}
      \Big[
      \sum_{i=1}^n \1\{c^{t+1}_i = 1\}
      \Big] \nonumber \\
    &\;=\;
      \E_{\calA, t}
      \Big[
      \sum_{i=1}^n
      \1\{c^{t+1}_i = 1\;\text{and}\; c^t_i =1 \}
      \Big]
      +
      \E_{\calA, t}
      \Big[
      \sum_{i=1}^n
      \1\{c^{t+1}_i = 1\;\text{and}\; c^t_i \neq 1 \}
      \Big]
      \;.
      \label{eq:drift-1}
  \end{align}
  For the first term of expression~\eqref{eq:drift-1},
  recall by the definition of the $\betaadopt$ protocol
  that if an agent chooses action 1 in round $t$,
  then in round $t+1$ it chooses action 1 with probability
  1 if it samples a neighbor that also chose action 1
  in round $t$, and it chooses action 1 with probability
  $1-\beta g^t_j$ if it sampled a neighbor that chose
  action $j \neq 1$ in round $t$.
  Thus we can write
  \begin{equation}
    \E_{\calA, t}
    \Big[
    \sum_{i=1}^n
    \1\{c^{t+1}_i = 1\;\text{and}\; c^t_i =1 \}
    \Big]
    \;=\;
    x^t_1 \Big(
    \frac{x^t_1}{n}  +
    \sum_{j \neq 1} \frac{x^t_j}{n} \big(1- \beta g^t_j\big)
    \Big) \;.
    \label{eq:drift-2}
  \end{equation}
  For the second term of expression~\eqref{eq:drift-1},
  recall again by the definition of the $\betaadopt$
  protocol that if an agent chooses some other
  action $j \neq 1$ in round $t$, then in round $t+1$
  it chooses action 1 with probability $\beta g^t_1$
  if it samples a neighbor that chose action 1 in round
  $t$, and with probability 0 otherwise.
  This means we can write
  \begin{equation}
    \E_{\calA, t}
    \Big[
    \sum_{i=1}^n
    \1\{c^{t+1}_i = 1\;\text{and}\; c^t_i \neq 1 \}
    \Big]
    \;=\;
    (n-x^t_1) \cdot\Big(\frac{x^t_1}{n} \cdot \beta g^t_1\Big)
    \;.
    \label{eq:drift-3}
  \end{equation}
  Substituting expressions~\eqref{eq:drift-2}
  and~\eqref{eq:drift-3} back into
  expression~\eqref{eq:drift-1} and taking expectation
  over the randomness of $\gt$, we then find 
  \begin{align}
    \E_{\calA, \pt}[x^{t+1}_1]
    &\;=\;
      x^t_1 \Big(
      \frac{x^t_1}{n}  +
      \sum_{j \neq 1} \frac{x^t_j}{n} \big(1- \beta \mu_j\big)
      \Big)
      +
      (n-x^t_1)\cdot\Big(\frac{x^t_1}{n} \cdot\beta\mu_1\Big)
    \nonumber \\
    &\;\ge\;
      x^t_1 \Big(
      \frac{x^t_1}{n}  +
      \Big(1 - \frac{x^t_1}{n}\Big)
      \cdot
      \big( 1- \beta \mu_2 \big)
      \Big)
      +
      (n-x^t_1)\cdot\Big(\frac{x^t_1}{n} \cdot\beta\mu_1\Big)
      \;,
      \label{eq:drift-4}
  \end{align}
  where the inequality follows from the ordering of
  the coordinates of $\bfmu$ in
  Assumption~\ref{ass:stationary-proof}.
  Further simplifying expression~\eqref{eq:drift-4}
  then yields
  \begin{equation}
    \E_{\calA, \pt}[x^{t+1}_1]
    \;\ge\;
    x^t_1 + x^t_1 \cdot \Big(1-\frac{x^t_1}{n}\Big)
    \cdot \beta(\mu_1 - \mu_2) \;,
    \nonumber 
  \end{equation}
  which means that
  \begin{equation}
    \E_{\calA, \pt}[x^{t+1}_1 - x^t_1]
    \;\ge\;
    x^t_1 \cdot \Big(1-\frac{x^t_1}{n}\Big)
    \cdot \beta(\mu_1 - \mu_2)
    \;.
    \label{eq:var-drift}
  \end{equation}

  Now conditioned on the event $\calA$, we have
  that $x^t_1 \ge n - 2\log^2 n \cdot \sqrt{n}$
  for all $\Tone \le t \le n$.
  Observe that the the right-hand side of
  expression~\eqref{eq:var-drift} is decreasing
  with $x^t_1$ and that
  $\Delta(n-1) \ge \beta (\mu_1-\mu_2) / 2$
  for sufficiently large $n$. 
  Then applying Theorem~\ref{thm:drift},
  it follows for sufficiently large $n$ that 
  \begin{align}
    \E_{\calA}[\hattau]
    &\;\le\;
      \frac{2}{\beta(\mu_1 - \mu_2)}
      \;+\;
      \int_{n-2\log^2 n \sqrt{n}}^{n-1}
      \frac{\beta (\mu_1 - \mu_2)}{
      s (1 - \tfrac{s}{n})
      }
      ds \\
    &\;\le\;
      \frac{2}{\beta(\mu_1 - \mu_2)}
      \;+\;
      \beta (\mu_1 - \mu_2)
      \int_{n/2}^{n-1}
      \frac{1}{
      s (1 - \tfrac{s}{n})
      }
      ds \\
    &\;\le\;
      \frac{2}{\beta(\mu_1 - \mu_2)}
      \;+\;
      \beta (\mu_1 - \mu_2) \cdot  \log (n-1)
      \;=\;
      O(\log n) \;.
  \end{align}
  Here, the second inequality comes from the fact
  that $s(1-s/n)$ is positive for all $1 \le n \le n-1$.
    
  Using Markov's inequality, we then have
  \begin{equation}
    \Pr_{\calA}\Big[
    \hattau \ge  \sqrt{n} \cdot \log n
    \Big]
    \;\le\;
    \frac{\E_{\calA}[\hattau]}{\sqrt{n} \cdot \log n}
    \;\le\;
    O\bigg(\frac{1}{\sqrt{n}}\bigg) \;.
  \end{equation}
  Finally, recalling that the event $\calA$ occurs
  with probability at least $1-O(1/n)$ we conclude
  using a union bound that
  $\hattau \le \sqrt{n} \cdot \log n$
  (and thus $\tconsensus \le \Ttwo$)
  with total probability at least
  $1- O(1/n + 1/\sqrt{n}) \ge 1 - O(1/\sqrt{n})$.
\end{proof}

Given Lemma~\ref{lem:t-consensus}, we can
now prove the bound on $\Rthree$ in
Lemma~\ref{lem:reg-phase-3}:
\begin{proof}[Proof (of Lemma~\ref{lem:reg-phase-3})]
  Recall that for $T \ge \Ttwo$ we define
  $\Rthree = \sum_{t = \Ttwo}^T
  \mu_1 - \E[ \langle \pt, \gt\rangle]$.
  Now by Lemma~\ref{lem:t-consensus}, we have
  with probability at least $1- O(1/\sqrt{n})$ that
  $\tconsensus \le \Ttwo$. Denote his event by $\calB$.
  Conditioning on $\calB$ means by definition
  that $p^t_1 = 1$ for all $t \ge \Ttwo$, and so
  $\sum_{t=\Ttwo}^T
  \mu_1 - \E_{\calB}[\langle \pt, \gt \rangle] = 0$.
  Thus $\Rthree = 0$ with probability at least $1-O(1/\sqrt{n})$.
\end{proof}

\subsection{Proof of Theorem \ref{thm:stationary-regret}}

Using the results of previous subsections,
we can now finish the proof of the overall
bound on the regret $R(T)$ from
Theorem~\ref{thm:stationary-regret}.

\begin{proof}
  By Lemmas~\ref{lem:reg-phase-1} and \ref{lem:reg-phase-2},
  we have under the assumptions of the theorem that
  $\Rone \le O(\sigma \log m)$ and
  $\Rtwo \le O(\log^3 n)$,
  respectively.
  Moreover, by Lemma~\ref{lem:reg-phase-3},
  we have with probability at least $1-O(1/\sqrt{n})$
  that $\Rthree = 0$.
  By definition, it follows that with
  probability at least $1-O(1/\sqrt{n})$:
  \begin{equation*}
    R(T)
    :=
    \Rone + \Rtwo + \Rthree
    \le
    O(\sigma \log m + \log^3 n)
    \;.
    \qedhere
  \end{equation*}
\end{proof}

\subsection{Proof of Theorem~\ref{thm:consensus}}

\begin{proof}
  The statement of Theorem~\ref{thm:consensus}
  follows directly from the proof of
  Lemma~\ref{lem:t-consensus}.
\end{proof}

% proof of adversarial regret

\section{Details on Regret Bounds for Adversarial Rewards}
\label{sec:adversarial-regret-details}

In this section, we develop the proof of
Theorem~\ref{thm:regret-adversarial} (restated below),
which gives a regret bound for the $\betaadopt$
protocol in the general, adversarial reward setting.
To start, we formally define the assumptions of
the adversarial reward setting:

\begin{restatable}[\textbf{Adversarial Rewards}]
  {assumption}{adversarialrewards}
  \label{reward:adversarial}
  For each $j \in [m]$, assume for all rounds $t \in [T]$
  that $g^t_j \sim \nu^t_j$, where
  for some $\sigma \ge 1$, 
  $\nu^t_j$ has support $[0, \sigma]$ and mean
  $\mu^t_j := \E[\nu_j] \in [0, 1]$.
  We assume that each $\mu^t_j$ can depend on
  the distribution $\pt$. 
\end{restatable}

Then to restate the claims of
Theorem~\ref{thm:regret-adversarial}:

\regretadversarial*

We remark that the main technical barriers in
allowing for longer horizons in our analysis stems
from the facts that
(i) the optimal setting of $\beta$ has a dependence
on $1/\sqrt{T}$ and (ii), we cannot leverage
any structure in the reward sequence to establish
a sufficient growth condition on the mass of
the highest-reward action in $\pt$. 
Thus for larger $T$, we cannot leverage the
epoch-based approach used in the proof of
Theorem~\ref{thm:stationary-regret},
and we resort to using a single application of
Proposition~\ref{prop:tau-step-regret} to obtain the result.

\begin{proof}[Proof (of Theorem~\ref{thm:regret-adversarial})]
  We will use the $T$ round regret bound from
  Proposition~\ref{prop:tau-step-regret}.
  For this, first recall from
  Proposition~\ref{prop:beta-adopt-params}
  that the $\betaadopt$ protocol induces a
  zero-sum family $\{\calF\}$ that satisfies
  Assumption~\ref{ass:F-params} with parameters
  $\alpha = 2\beta$, $\delta =0$, and $L=2$,
  for $\beta \in (0, \min(1/4, 1/\sigma)]$.
  Moreover, setting $\beta := \sqrt{(\log m)/T}$
  ensures that the constraint on $\beta$ is
  satisfied for sufficiently large $n$.

  Then setting $\kappa := 3 + L = 5$
  we apply the bound from
  Proposition~\ref{prop:tau-step-regret}
  with $c = 1$ and
  $\alpha = 2\beta = \Theta(\sqrt{(\log m)/T})$.
  This yields
  \begin{align}
    R(T)
    &\;\le\;
      O\bigg(
      \frac{\log m}{\beta}
      + \beta T
      \bigg)
      +
      O\bigg(
      \frac{\sigma \kappa^T \cdot \sqrt{m \log n}}{\sqrt{n}}
      +
      \frac{\sigma T}{n^2}
      \bigg) \nonumber \\
    &\;\le\;
      O\bigg(
      \frac{\log m}{\beta}
      + \beta T
      \bigg)
      +
      O\bigg(
      \frac{\sigma \sqrt{m \log n}}{ n^{\eps}}
      +
      \frac{\sigma T}{n^2}
      \bigg)
      \;.
      \label{eq:adv-1}
  \end{align}
  Here, the final line comes from the
  assumption that $T \le (0.5-\eps) \log_{\kappa} n$
  for some $\eps \in (0, 0.5)$,
  and thus $\kappa^T/(\sqrt{n}) \le O(1/n^{\eps})$.
  Moreover, by the bound on $T$, and assuming
  $m = o(\log^{1.5}n)$, it follows that
  the second term in~\eqref{eq:adv-1}
  is at most $O(\sqrt{T\log m})$.
  Finally, using the setting
  $\beta := \Theta(\sqrt{(\log m) /T})$,
  the first term in~\eqref{eq:adv-1}
  is also bounded by $O(\sqrt{T\log m})$.
  Thus we conclude
  $R(T) \le O(\sqrt{T \log m})$.
\end{proof}

\subsection{Regret Lower Bound for $\betaadopt$ Protocol}
\label{sec:adversarial-regret-details:lower}

In light of the constraint on $T$ in
the adversarial regret bound of
Theorem~\ref{thm:regret-adversarial},
we prove here that in this adversarial reward setting,
the $\betaadopt$ algorithm cannot do much better.
In particular, we show that there exist reward sequences
such that, for $T \ge \Omega(\sqrt{n} \log^3 n)$,
$R(T) \ge \Omega(T)$. 

\begin{proposition}
  \label{prop:adversarial-lower}
  Let $R(T)$ be the regret of the $\betaadopt$ protocol
  with $\beta \le (0, \min(1/4, 1/\sigma)$
  an absolute constant. Then there exists
  an adversarial reward sequence such that,
  for $n$ sufficiently large
  and $T \ge \Omega(\sqrt{n} \log^3 n)$,
  $R(T) \ge \Omega(T)$.  
\end{proposition}

\begin{proof}
  The proof of the proposition relies on
  generating a reward sequence that
  is stationary for the first $O(\sqrt{n} \log n)$ rounds,
  and then ``switches'' to an adversarial
  sequence once $\pt$ has reached consensus
  on one coordinate. 
  For simplicity, consider the case where $m=2$,
  and note that the following construction can
  easily be lifted to any larger $n$.
  Let $\tau_1 := 2 \sqrt{n}\log n$ (i.e.,
  the same value of $\Tone$ from
  the setup of Theorems~\ref{thm:stationary-regret}
  and~\ref{thm:consensus}.
  Deterministically define $\gt := (1, 0)$
  for all $0 \le t \le \Tone$ and
  $\gt := (0, 1)$ for all $t > \Tone$.

  Now observe by Theorem~\ref{thm:consensus}
  that in the stationary reward setting,
  with probability at least $1-O(1/\sqrt{n})$,
  the $\betaadopt$ protocol reaches
  best-action consensus within
  $\Tone$ rounds. Let $\calA$ be the event
  that this claim holds.
  Then for all $t > \Tone$, we must have by
  construction that 
  $\E_{\calA}[g^t_2 - \langle\pt, \gt\rangle] = 1$.
  Moreover, for any $T \ge 2 \sqrt{n} \log n$
  action 2 must have maximal cumulative
  mean reward, again by construction.
  It follows that for this range of $T$ we have
  \begin{equation*}
    R(T)
    \;=\;
    \sum_{t \in [T]}
    g^t_2 - \E[\langle \pt, \gt \rangle]
    \;\ge\;
    (T - \Tone)(1-(1/\sqrt{n}))
    \;\ge\;
    \Omega(T) \;,
  \end{equation*}
  which concludes the proof.
\end{proof}

\bibliographystyle{alpha}
\bibliography{references}

\end{document}